\documentclass{article}

 \PassOptionsToPackage{numbers, compress}{natbib}
\usepackage[final]{neurips_2024}

\RequirePackage{algorithm}
\RequirePackage{algorithmic}
\usepackage{subcaption}
\usepackage{tikz}

\usepackage[utf8]{inputenc} %
\usepackage[T1]{fontenc}

\usepackage{url}
\usepackage{hyperref}

\usepackage{amsmath}
\usepackage{amssymb}
\usepackage{mathtools}
\usepackage{amsthm}
\usepackage{amsfonts}

\usepackage{xcolor}         %

\usepackage{xparse}

\usepackage{microtype}
\usepackage{graphicx}
\usepackage{booktabs}
\usepackage{nicefrac} 

\usepackage{silence}
\usepackage{times}
\usepackage{graphicx}
\usepackage{comment}
\usepackage{color}
\usepackage[normalem]{ulem}
\usepackage{scrextend}
\usepackage{multirow}
\usepackage{cancel}
\usepackage{xspace}
\usepackage{mathtools}
\usepackage{bm}
\usepackage{bbm}
\usepackage{euscript}
\usepackage{mleftright}
\usepackage{thm-restate}
\usepackage{enumitem}

\usepackage[capitalize,noabbrev]{cleveref}

\newcommand{\OMIT}[1]{}
\newcommand{\term}[1]{\ensuremath{\texttt{#1}}\xspace}

\newcommand{\indicator}[1]{\mathbbm{1}\mleft[#1\mright]}

\newcommand{\initOneLiners}{%
 	\setlength{\itemsep}{0pt}
	\setlength{\parsep }{0pt}
  	\setlength{\topsep }{0pt}     	
}

\usepackage{framed}

\def \OPT {\term{Opt}}
\def \OGD {\term{OGD}}
\def \EXPSIX {\term{EXP3-SIX}}
\def \optadv {\OPT_{\term{Adv}}}
\def \optstoc {\OPT_{\term{Stoc}}}
\def \AlgP {\term{Alg}_\term P}
\def \AlgD {\term{Alg}_\term D}
\def \uP {u^\term P}
\def \uD {u^\term D}
\def \RD {R^\term D}
\def \RP {R^\term P}
\def \rhoadv {\rho_\term{Adv}}
\def \rhostoc {\rho_\term{Stoc}}
\def \Rew {\term{Rew}}

\newcommand{\xhdr}[1]{\vspace{2mm} \noindent{\bf #1}}

\newcommand{\ie}{{\em i.e.,~\xspace}}
\newcommand{\eg}{{\em e.g.,~\xspace}}

\newcommand{\Reals} {\ensuremath{\mathbb{R}}} %reals
\newcommand{\cA} {\ensuremath{\mathcal{A}}}

\newcommand{\cD} {\ensuremath{\mathcal{D}}}

\newcommand{\cF} {\ensuremath{\mathcal{F}}}

\newcommand{\cL} {\ensuremath{\mathcal{L}}}

\newcommand{\cO} {\ensuremath{\mathcal{O}}}

\NewDocumentCommand{\cTG}{o} {\IfNoValueTF {#1}{\ensuremath{\mathcal{T}_\texttt{G}}} {\ensuremath{\mathcal{T}_{\texttt{G},#1}}}
}
\newcommand{\cX} {\ensuremath{\mathcal{X}}}
\newcommand{\cZ} {\ensuremath{\mathcal{Z}}}

\makeatletter
\newcommand{\commentsymbol}{\it\color{gray}$\triangleright$~}
\newcommand{\LComment}[1]{{\commentsymbol{#1}}}
\makeatother

\renewcommand{\vec}[1]{\bm{#1}}
\newcommand{\vx}{{\vec{x}}}
\newcommand{\vy}{\vec{y}}
\newcommand{\vc}{\vec{c}}
\newcommand{\vg}{\vec{g}}

\newcommand{\vzero}{\vec{0}}

\newcommand{\vlambda}{\vec{\lambda}}

\newcommand{\defeq}{\coloneqq}

\def \BIDDER {\term{B}}

\newcommand{\eR}{\EuScript{R}}

\newcommand{\E}{\mathbb{E}}

\newcommand{\cR}{\eR}

\NewDocumentCommand{\cRd}{O{}}{\ensuremath{\mathcal{A}^{\texttt{D}}_{#1}}}

\NewDocumentCommand{\observeutility}{O{\ell_t}}{\ensuremath{\textsc{ObserveUtility}(#1)}}
\NewDocumentCommand{\lossp}{O{t}}{\ensuremath{u^\texttt{P}_{#1}}}
\NewDocumentCommand{\lossd}{O{t}}{\ensuremath{u^\texttt{D}_{#1}}}
\NewDocumentCommand{\lossb}{O{t}}{\ensuremath{\ell^\BIDDER_{#1}}}
\NewDocumentCommand{\cump}{O{T}O{\texttt{P}}}{\ensuremath{{\cR}^{#2}_{#1}}}
\NewDocumentCommand{\uppp}{O{T}O{\texttt{P}}}{\ensuremath{{\EuScript{E}}^{#2}_{#1}}}
\NewDocumentCommand{\cumd}{O{T}}{\ensuremath{\cR^\texttt{D}_{#1}}}
\NewDocumentCommand{\cumdr}{O{T}}{\ensuremath{\EuScript{E}^{\texttt{D},\texttt{R}}_{#1}}}
\NewDocumentCommand{\cumdb}{O{T}}{\ensuremath{\EuScript{E}^{\texttt{D},\texttt{B}}_{#1}}}
\NewDocumentCommand{\cume}{O{T,\delta}O{}}{\ensuremath{\EuScript{E}_{#1}^{#2}}}
\NewDocumentCommand{\cumed}{O{T,\delta/d}O{}}{\ensuremath{\EuScript{E}_{#1}^{#2}}}
\NewDocumentCommand{\cumg}{O{T,\delta}}{\ensuremath{\EuScript{E}_{#1}}}
\NewDocumentCommand{\regp}{O{T}}{\ensuremath{R^\texttt{P}_{#1}}}
\NewDocumentCommand{\regd}{O{T}}{\ensuremath{R^\texttt{D}_{#1}}}
\newcommand{\err}{\term{Err}}
\usepackage{stmaryrd}

\newcommand{\range}[1]{{\llbracket {#1} \rrbracket}}

\usepackage{pifont}
\usepackage[normalem]{ulem} % for \sout

\definecolor{mygreen}{rgb}{0.0, 0.5, 0.0}
\definecolor{myorange}{rgb}{0.55, 0.62, 1}

\DeclareMathOperator*{\arginf}{arg\,inf}

\newcommand{\insta}{\term{A}}
\newcommand{\instb}{\term{B}}

\newcommand{\gaplambda}{\widetilde{\Delta}}

\newcommand{\bwk}{\term{BwK}}

\newcommand{\cbwlc}{\term{CBwLC}}

\newcommand{\lagrangebwk}{\term{LagrangeBwK}}
\newcommand{\cb}{\term{CB}}
\theoremstyle{plain}
\newtheorem{theorem}{Theorem}[section]

\newtheorem{example}[theorem]{Example}
\newtheorem{lemma}[theorem]{Lemma}
\newtheorem{claim}[theorem]{Claim}
\newtheorem{corollary}[theorem]{Corollary}
\theoremstyle{definition}
\newtheorem{definition}[theorem]{Definition}
\newtheorem{assumption}[theorem]{Assumption}
\newtheorem{remark}[theorem]{Remark}

\usepackage{amsmath}
\usepackage{amssymb}
\usepackage{mathtools}
\usepackage{amsthm}
\usepackage{wrapfig}

\usepackage[capitalize,noabbrev]{cleveref}

\usepackage[utf8]{inputenc} % allow utf-8 input
\usepackage[T1]{fontenc}    % use 8-bit T1 fonts
\usepackage{hyperref}       % hyperlinks
\usepackage{url}            % simple URL typesetting
\usepackage{booktabs}       % professional-quality tables
\usepackage{amsfonts}       % blackboard math symbols
\usepackage{nicefrac}       % compact symbols for 1/2, etc.
\usepackage{microtype}      % microtypography
\usepackage{xcolor}         % colors

\definecolor{niceRed}{RGB}{190,38,38}
\definecolor{Red2}{RGB}{219, 50, 54}
\definecolor{mgreen}{HTML}{9ECA8D}
\definecolor{blueGrotto}{HTML}{059DC0}
\definecolor{limeGreen}{HTML}{00CC00}
\definecolor{myellow}{rgb}{0.88,0.61,0.14}
\definecolor{navyBlueP}{HTML}{03468F}
\definecolor{Sepia}{HTML}{7F462C}
\definecolor{red2}{HTML}{CC0000}
\definecolor{orange2}{HTML}{FF8000}
\definecolor{mgray}{HTML}{ABB3B8}
\definecolor{myPurple}{RGB}{175,0,124}
\definecolor{royalBlue}{HTML}{057DCD}
\definecolor{mpink}{HTML}{FC6C85}

\hypersetup{
	citebordercolor = limeGreen,
	linkbordercolor = red2
}

\title{No-Regret is not enough! Bandits with General Constraints through Adaptive Regret Minimization}

\author{
	Martino Bernasconi$^\dagger$ \quad
	Matteo Castiglioni$^\ddagger$ \quad
	Andrea Celli$^\dagger$\\
	$^\dagger$\ Bocconi university\\
	$^\ddagger$\ Politecnico di Milano\\
	{\textcolor{black}{\scriptsize\texttt{\{martino.bernasconi,andrea.celli2\}@unibocconi.it}, \quad \texttt{matteo.castiglioni@polimi.it}}}
}

\begin{document}
	
	\maketitle

\begin{abstract}
	In the bandits with knapsacks framework (\bwk) the learner has $m$ resource-consumption (\ie packing) constraints. We focus on the generalization of \bwk in which the learner has a set of general long-term constraints. The goal of the learner is to maximize their cumulative reward, while at the same time achieving small cumulative constraints violations.
	In this scenario, there exist simple instances where conventional methods for \bwk fail to yield sublinear violations of constraints.
	We show that it is possible to circumvent this issue by requiring the primal and dual algorithm to be \emph{weakly adaptive}. Indeed, even in absence on any information on the Slater's parameter $\rho$ characterizing the problem, the interplay between weakly adaptive primal and dual regret minimizers yields a ``self-bounding'' property of dual variables. In particular, their norm remains suitably upper bounded across the entire time horizon even without explicit projection steps. 
	By exploiting this property, we provide \emph{best-of-both-worlds} guarantees for stochastic and adversarial inputs. In the first case, we show that the algorithm guarantees sublinear regret. In the latter case, we establish a tight competitive ratio of $\rho/(1+\rho)$. In both settings, constraints violations are guaranteed to be sublinear in time.
	Finally, this results allow us to obtain new result for the problem of \emph{contextual bandits with linear constraints}, providing the first no-$\alpha$-regret guarantees for adversarial contexts.
\end{abstract}

\section{Introduction}

We consider a problem in which a decision maker tries to maximize their cumulative reward over a time horizon $T$, subject to a set of $m$ \emph{long-term constraints}. At each round $t$, the learner chooses $x_t\in\cX$ and, subsequently, observes a reward $f_t(x_t)\in [0,1]$ and $m$ constraint functions $\vg_t(x_t)\in[-1,1]^m$. Then, the problem becomes that of finding a sequence of decisions which guarantees a reward close to that of the best fixed decision in hindsight, while satisfying long-term constraints $\sum_{t=1}^T \vg_t(\vx_t)\le\mathbf{0}$ up to small sublinear violations.  
This framework subsumes the \emph{bandits with knapsacks} (\bwk) problem, where there are only resource-consumption constraints \cite{Badanidiyuru2018jacm,agrawal2019bandits,immorlica2022jacm}.

% \ac{improve: general constraints--> as a byproduct so fare contextual}

% This model can be used to describe the \emph{contextual bandits with linear constraints} (\cbwlc) problem, which was recently studied by \citet{slivkins2023contextual} in the context of stochastic and non-stationary environments. A natural question is the following: \emph{can we establish new guarantees for \cbwlc when contexts are generated by an adversary?} This would close the gap between what can be achieved in contextual bandits in absence of constraints \citep{foster2020beyond}, and what can be achieved in constrained scenarios. 

Inputs $(f_t,\vg_t)$ may be either stochastic or adversarial. The goal is designing algorithms providing guarantees for both input models, without prior knowledge of the specific environment they will encounter. Achieving this goal involves addressing two crucial challenges which prevent a direct application of primal-dual approaches based on the \lagrangebwk framework in \cite{immorlica2022jacm}.

% . First, all black-box primal-dual approaches based on the \lagrangebwk framework by \citet{immorlica2019adversarial,immorlica2022jacm} require knowledge of the Slater's parameter characterizing a strictly feasible solution of the problem. Second, there exists simple instances of our problem in which, even when Slater's parameter is known, existing Lagrangian-based solutions fail to guarantee sublinear constraints violations. 

% When $(f_t,\vg_t)$ are selected by an adversary, there are two crucial complications which prevent a direct application of Lagrangian based approaches based on the \lagrangebwk framework by \citet{immorlica2019adversarial,immorlica2022jacm}.

\subsection{Technical Challenges}

%\xhdr{Technical challenge 1.} 

In order to obtain meaningful regret guarantees, primal-dual frameworks based on \lagrangebwk need to control the magnitude of dual variables. This is necessary as dual variables appear in the loss function of the primal algorithm, and, therefore, influence the no-regret guarantees provided by the primal algorithm.  
In the context of knapsack constraints, this is usually achieved by exploiting the existence of a strictly feasible solution with Slater's parameter $\rho$, consisting of a \emph{void action} which yields zero reward and resource consumption. For instance, the frameworks of \cite{balseiro2022best,castiglioni2022online} guarantee boundedness of dual multipliers through an explicit projection step on the interval $[0,1/\rho]$.
However, in settings with general constraints beyond resource consumption, it is often unreasonable to assume that the learner knows the Slater's parameter $\rho$ a priori. The problem of operating without knowledge of $\rho$ has been already addressed in the stochastic setting \cite{agrawal2014bandits,agrawal2019bandits,yu2017online,wei2020online,castiglioni2022unifying}. For instance, a simple approach for the case of stochastic inputs involves adding an initial estimation phase to calculate an estimate of $\rho$, and subsequently treating this estimate as the true parameter \citep{castiglioni2022unifying}. However, these techniques cannot be applied in adversarial environments as estimates of $\rho$ based on the initial rounds could be inaccurate about future inputs.
% as they would require some knowledge about future inputs. 

%\xhdr{Technical challenge 2.} 
Primal-dual templates based on \lagrangebwk usually operate under the assumption that the primal and dual algorithms have the no-regret property. In the case of standard \bwk, the no-regret requirement is sufficient to obtain optimal guarantees (see, \eg \cite{immorlica2022jacm,castiglioni2022online}). 
However, in our model, there exist simple instances in which the primal and dual algorithms satisfy the no-regret requirement, but the overall  framework fails to guarantee small constraints violations (see \Cref{sec:example no regret not enough}).
Moreover, known techniques to prevent this problem, such as introducing a \emph{recovery phase} to prevent excessive violations, crucially require a priori knowledge of the Slater's parameter $\rho$ \citep{castiglioni2022unifying}.

\subsection{Contributions}

Our approach is based on a generalization of the technique presented in \cite{castiglioni2023online} for online bidding under one budget and one return-on-investments constraint. The crux of the approach is requiring that both the primal and dual algorithms are \emph{weakly adaptive}, that is, they guarantee a regret upper bound of $o(T)$ for each sub-interval of the time horizon \cite{hazan2007adaptive}. We generalize this approach to the case of $m$ general constraints, thereby providing the first primal-dual framework for this problem that can operate without any knowledge of Slater's parameter in both stochastic and adversarial environments.

First, we prove a ``self-bounding'' lemma for the case of $m$ arbitrary constraints. It shows that, if the primal and dual algorithms are weakly adaptive, then boundedness of dual multipliers emerges as a byproduct of the interaction between the primal and dual algorithm. Thus, it is possible to guarantee a suitable upper bound on the dual multipliers even without any information on Slater's parameter.  

We use this result to prove \emph{best-of-both-worlds} no-regret guarantees for primal-dual frameworks derived from \lagrangebwk which employ weakly adaptive primal and dual algorithms. 
Our guarantees will be modular with respect to the regret guarantees of the primal and dual algorithms. In presence of a suitable primal regret minimizer, we show that our framework yields the following no-regret guarantees while attaining sublinear constraints violations: in the stochastic setting, it guarantees sublinear regret with respect to the best fixed randomized strategy that is feasible in expectation. Remarkably, this result is obtained without having to allocate the initial $T^{1/2}$ rounds for estimating the unknown parameter as in \cite{castiglioni2022unifying}.
In the adversarial setting, our framework guarantees a competitive ratio of $\rho/(1+\rho)$ against the best unconstrained strategy in hindsight. We provide a lower bound showing that this cannot be improved if constraint violations have to be $o(T)$.
This is the first regret guarantee for our problem in adversarial environments.   

Finally, we show that our model can be used to describe the \emph{contextual bandits with linear constraints} (\cbwlc) problem, which was recently studied by \cite{slivkins2023contextual, han2023optimal} in the context of stochastic and non-stationary environments. Our framework allows to extend these works in two directions: we establish the first no-$\alpha$-regret guarantees for \cbwlc when contexts are generated by an adversary, and we provide the first $\widetilde{O}(\sqrt{T})$ guarantees for the stochastic setting when the learner does not know an estimate of the Slater's parameter of the problem.

\section{Related Work}

\xhdr{Bandits with Knapsacks.} The (stochastic) \bwk problem was introduced an optimally solved by \cite{badanidiyuru2013bandits,Badanidiyuru2018jacm}. Other algorithms with optimal regret guarantees have been proposed by \cite{agrawal2014bandits,agrawal2019bandits}, whose approach is based on the paradigm of \emph{optimism in the face of uncertainty}, and in \cite{immorlica2019adversarial,immorlica2022jacm}. In the latter works, the authors propose the \lagrangebwk framework, which has a natural interpretation: arms can be thought of as primal variables, and resources as dual variables. The framework works by setting up a repeated two-player zero-sum game between a primal and a dual player, and by showing convergence to a Nash equilibrium of the expected Lagrangian game.  

\xhdr{Adversarial \bwk.} The adversarial \bwk problem was first introduced in \cite{immorlica2019adversarial,immorlica2022jacm}, where they studied the case in which the learner has $m$ knapsack constraints, and inputs are selected by an oblivious adversary. Their algorithm is based on a modified analysis of \lagrangebwk, and guarantees a $O(m\log T)$ competitive ratio. Subsequently, \cite{kesselheim2020online} provided a new analysis obtaining a $O(\log m\log T)$ competitive ratio, which is optimal. In the case in which budgets are $\Omega(T)$, \cite{castiglioni2022online} showed that it is possible to achieve a constant competitive ratio of $1/\rho$ where $\rho$ is the per-iteration budget. 

\xhdr{Beyond packing constraints.} \cite{castiglioni2022online} studies a setting with general constraints analogous to ours, and show how to adapt the \lagrangebwk framework to obtain best-of-both-worlds guarantees when Slater's parameter is known a priori. Similar guarantees are also provided, in the stochastic setting, by \cite{slivkins2023contextual}, which then extend the results to the \cbwlc model. Finally, the work of \cite{castiglioni2023online} introduces the use of weakly adaptive regret minimizers within the \lagrangebwk framework, and provides guarantees in the specific case of one budget constraint and one return-on-investments constraint.
% (\ie for only one unknown Slater's parameter).

\xhdr{Contextual bandits (\cb).} We briefly survey the most relevant works for our paper. Further references can be found in \cite[Chapter 8]{Slivkins2019intro}. As in \cite{SlivkinsSF23}, we focus on \cb with regression oracles \citep{foster2018practical,foster2020beyond,bietti2021contextual,simchi2022bypassing}.
The contextual version of \bwk was first studied by \cite{badanidiyuru2014resourceful} in the case of classification oracles. A regret-optimal and oracle-efficient algorithm for this problem was proposed by \cite{agrawal2016efficient} by exploiting the oracle-efficient algorithm for \cb by \cite{agarwal2014taming}. The first regression-based approach for constrained \bwk was proposed by \cite{agrawal2016linear} by exploiting the optimistic approach for linear \cb \cite{li2010contextual,chu2011contextual,abbasi2011improved}.
\cite{han2023optimal} propose a regression-based approach for a  constrained \bwk setup under stochastic inputs.
Finally, a notable special case of constrained \cb is online bidding under constraints \citep{balseiro2019learning,celli2023best,gaitonde2023budget,feng2023online,wang2023learning}.

\xhdr{Other related works.} 
\cite{fikioris2023approximately} show how to interpolate between the fully stochastic and the fully adversarial setting, depending on the magnitude of fluctuations in expected rewards and consumptions across rounds.
\cite{liu2022non} study a non-stationary setting and provide no-regret guarantees against the best dynamic policy through a UCB-based algorithm.
Some recent works explore the case in which resource consumptions in \bwk can be non-monotonic \citep{kumar2022non,bernasconi2023bandits,longVersion}.
Finally, a related line of works is the one on online allocation problems with fixed per-iteration budget, where the input pair of reward and costs is observed \emph{before} the learner makes a decision \citep{balseiro2022best,Balseiro2023}.

\section{Preliminaries} \label{sec:prelim}
There are $T$ rounds and $m$ constraints. We denote with $\cX\subset\Reals^K$ the decision space of the agent.
At each round $t\in\range{T}$, the agent selects an action $x_t\in\cX$ and subsequently observes a reward $f_t(x_t)$ and costs function $\vg_t(x_t)\in[-1,1]^m$, with $f_t:\cX\to[0,1]$ and $g_{t,i}:\cX\to[-1,1]$ for each $i\in\range{m}$.\footnote{In this work, for any $a,b\in\mathbb{N}$, with $a<b$ we denote with $\range{a}$ the set $\{1,\ldots, a\}$ while $\range{a,b}$ the set $\{a+1,\ldots, b\}$.} The reward and cost functions can either be chosen by an oblivious adversary or drawn from a distribution.
%
% When the reward and the costs are stochastic we denote by $\bar f$ and $\bar \vg$ the mean of $f_t$ and $\vg_t$, respectively, and the rewards are drawn so that $\mathbb{E}_{\term {Env}}[f_t(x)]=\bar f(x)$ (and similarly for the costs), where $\mathbb{E}_{\term {Env}}$ denote expectation over the environment measure.
%
The goal of the decision maker is to maximize the cumulative reward $\Rew(T)\coloneqq\sum_{t\in\range{T}} f_t(x_t)$, while minimizing the cumulative violation $V_i(T)$ defined as 
\[\textstyle{
V_i(T)\coloneqq\sum_{t\in\range{T}} g_{t,i}(x_t)}
\]
for each constraint $i\in\range{m}$. We denote by $V(T)\defeq \max_{i\in\range{m}}V_i(T)$ the maximum cumulative violation across the $m$ constraints.

\subsection{Baselines}
We will provide best-of-both-worlds no-regret guarantees for our algorithm, meaning that it achieves optimal theoretical guarantees both in the stochastic and adversarial setting. In this section, we introduce the baselines used to define the regret in these two scenarios. 

\paragraph{Adversarial Setting} In the adversarial setting we employ the strongest baseline possible, \ie the best \emph{unconstrained} strategy in hindsight:
\[\textstyle
\optadv\coloneqq\sup_{x\in\cX}\sum_{t\in\range{T}}f_t(x).
\]
This baseline is more powerful than the best fixed strategy which is feasible on average \citep{immorlica2022jacm,castiglioni2022online}, which is the most common baseline in the literature. Our algorithm will yield an optimal competitive ratio against this stronger baseline. 
 In this setting, we define $\rhoadv$ as the feasibility parameter of the problem instance, \ie the largest reduction of cumulative violations that the agent is guaranteed to achieve by playing a ``safe'' strategy $\xi^\circ\in\Delta(\cX)$, where $\Delta(\cX)$ is the set of all probability measures on $\cX$. Formally, 
 \[
 \rhoadv:=-\max_{t\in\range{T}, i\in\range{m}}\mathbb{E}_{x\sim\xi^\circ}[g_{t,i}(x)] \quad\text{and}\quad
\xi^\circ:=\arginf_{\xi\in\Delta(\cX)}\max_{t\in\range{T}, i\in\range{m}}\mathbb{E}_{x\sim\xi}[g_{t,i}(x)].\] 
\paragraph{Stochastic Setting} 
When the reward and the costs are stochastic we denote by $\bar f$ and $\bar \vg$ the mean of $f_t$ and $\vg_t$, respectively. In particular, we have that the rewards are drawn so that $\mathbb{E}_{\term {Env}}[f_t(x)]=\bar f(x)$ (and similarly for the costs), where $\mathbb{E}_{\term {Env}}$ denotes expectation over the environment measure.
We define the baseline for the stochastic setting as the best fixed \emph{randomized} strategy that satisfies the constraints in expectation, which is the standard choice in Stochastic Bandits with Knapsacks settings \cite{badanidiyuru2013bandits,immorlica2022jacm}.
Formally, 
\[
\optstoc\coloneqq\sup\limits_{\xi\in\Delta(\cX):\, \mathbb{E}_{x\sim\xi}[\bar\vg(x)]\le \vzero} \mathbb{E}_{x\sim \xi}[\bar f(x)].
\]
Similarly to the adversarial case, we define the feasibility parameter $\rhostoc$ as the ``most negative'' cost achievable by randomized strategies \emph{in expectation}:
\[
\rhostoc:= -\inf\limits_{\xi\in\Delta(\cX)}\max_{ i\in\range{m}}\mathbb{E}_{x\sim\xi} [\bar g_{i}(x)].
\]
As it is customary in relevant literature (see, \eg \cite{immorlica2022jacm,castiglioni2022online,castiglioni2022unifying}), we make the following natural assumption about the existence of a strictly feasible solution.  Note that we do not make any assumption on the variance of the samples $(f_t,\vg_t)$ as we assume that they have bounded support, \ie with probability holds that $f_t(x)\in[0,1]$ and $g_{t,i}(x)\in[-1,1]$ for all $x\in\cX$ and $i\in\range{m}$.
\begin{assumption}
	In the adversarial setting, the sequence of inputs $(f_t,\vg_t)_{t=1}^T$ is such that $\rhoadv>0$. In the stochastic setting, the environment \term{Env} is such that $\rhostoc>0$. 
\end{assumption}

\begin{remark}
	We will describe a best-of-both-worlds type algorithm, that attains optimal guarantees both under stochastic and adversarial inputs, without knowledge of the specific setting in which the algorithm operates. It should be noted that $\rhoadv$ and $\rhostoc$ are \emph{not} known by the algorithm. 
	While the algorithm could potentially efficiently estimate $\rhostoc$ in stochastic settings, as shown in \cite{castiglioni2022unifying}, acquiring knowledge of $\rhoadv$ in the adversarial setting would necessitate information about future inputs. This requirement is generally unfeasible for most instances of interest.
\end{remark}

\section{%Best-of-Both-Worlds Guarantees and Lower Bounds
On Best-Of-Both-Worlds Guarantees}\label{sec:best of both worlds}

We employ the expression \emph{best-of-both-worlds} as defined in \cite{balseiro2022best} for the case of online allocation problems with resource-consumption constraints. In this context, we expect different types of guarantees depending on the input model being considered.

%For simplicity, let us assume that $\cX$ is finite.
When inputs are stochastic, a best-of-both-worlds algorithm should guarantee that, given failure probability $\delta>0$, with probability at least $1-\delta$  
\[
\max(\optstoc-\Rew(T), V(T))=\widetilde O(\sqrt{T}).
\]
The dependency on $T$ is optimal since, in the worst case, it is optimal even without constraints \cite{auer2002nonstochastic}.

In adversarial settings, a best-of-both-worlds algorithm should guarantee that, with probability at least $1-\delta$, 
\[
\max\left(\optadv-\alpha \Rew(T), V(T)\right)=\widetilde O(\sqrt{T}),
\]
where $\alpha>1$ is the \emph{competitive ratio}.  
In the \bwk scenario with only resource-consumption constraints, the optimal competitive ratio attainable is $\alpha=1/\rhoadv$. In that setting, $\rhoadv$ denotes the per-iteration budget, which we can assume is equal for each resource without loss of generality.
In our set-up, considering arbitrary and potentially negative constraints, we will present an algorithm for which the above holds for $\alpha:=1+1/\rhoadv$. 
The following result shows that this competitive ratio is optimal. In particular, we show that it is not possible to obtain cumulative constraint violations of order $o(T)$ and competitive ratio strictly less that $1+1/\rhoadv$ (omitted proofs can be found in the Appendix).

%\ac{make expectations more precise?}
\begin{restatable}{theorem}{theoremtightCR}[Lower bound adversarial setting]\label{prop:tighcompration}
	Consider the family of all adversarial instances with $\cX=\{a_1, a_2\}$, each characterized by a parameter $\rhoadv$ and optimal reward $\optadv$. Then, no algorithm can achieve, on all instances, sublinear cumulative violations $\E[V(T)]=o(T)$ and
	\[\frac{\optadv}{\E[\Rew]}>1+\frac{1}{\rhoadv}\]
\end{restatable}
%
% \begin{theorem}[Lower bound adversarial setting]\label{prop:tighcompration}
% 	For all $\epsilon>0$ and $\delta\in(0,1)$, there exists two instances in which it is impossible to obtain $V(T)\le\epsilon T$ and 
% 	\[\frac{\optadv}{\E[\Rew]}\ge\frac{1+\rhoadv}{\rhoadv(1+\delta)+2\epsilon}\] in both instances.
% \end{theorem}

\section{Lagrangian Framework}\label{sec:Lagrangian}
\begin{wrapfigure}[15]{r}{0.5\textwidth}
	%\vspace{-1.1cm}
	\vspace{-0.65cm}
\begin{minipage}{0.5\textwidth}
\begin{algorithm}[H]
	\caption{Primal-Dual Algorithm}
	\label{alg:alg1}
	\begin{algorithmic}[1]
		\STATE {\bfseries Input:} $\AlgP$ and $\AlgD$.
		\FOR{$t = 1, 2, \ldots , T$}
		\STATE{\bfseries Primal decision:} $x_t \gets \AlgP$		
		\STATE {\bfseries Dual decision:} $\vlambda_t \gets \AlgD$
		\STATE {\bfseries Observe:} $f_t(x_t)$ and $\vg_t(x_t)$
		\STATE {\bfseries Primal update:} feed $\uP_t(x_t)$ to $\AlgP$, where
		\begin{ALC@g}
			\STATE $\uP_t(x_t)\gets  f_t(x_t)  -   \langle \vlambda_t, \vg_t(x_t)\rangle$ %\LComment{bandit feedback}
		\end{ALC@g}
		\STATE {\bfseries Dual update:} 
		\begin{ALC@g}
			\item Feed $\uD_t:\vlambda\mapsto- f_t(x_t)+\langle\vlambda,\vc_t(x_t)\rangle$ to $\AlgD$% \LComment{full feedback}
		\end{ALC@g}
		\ENDFOR
	\end{algorithmic}
\end{algorithm}
\end{minipage}
\end{wrapfigure}

%
% We will consider the Lagrangian function related with the optimization problems above. 
Given the reward function $f:\cX\to[0,1]$ and the costs functions $\vg:\cX\to[-1,1]^m$ we define the Lagrangian $\cL_{f,\vg}:\cX\times\Reals^m_+\to\Reals$ as:
\[
\cL_{f,\vg}(x, \vlambda)\coloneqq f(x) - \langle\vlambda, \vg(x)\rangle.
\]
We will consider a modular primal-dual approach that employs a \emph{primal} algorithm $\AlgP$, producing primal decisions $x_t$, and a \emph{dual} algorithm $\AlgD$ that produces dual decisions $\vlambda_t$ for all $t$. We assume that $\AlgP$ and $\AlgD$ produce their decisions in order to maximize their utilities $\uP_t$ and $\uD_t$, respectively. We define $\uP_t:x\mapsto\cL_{f_t,\vg_t}(x, \vlambda_t)$ and $\uD_t:\vlambda\mapsto-\cL_{f_t,\vg_t}(x_t,\vlambda)$. 
The regret of the primal algorithm $\AlgP$ on any subset $I\subseteq\range{T}$ is defined as:
\[
\RP_I( \cX):=\sup_{x\in\cX}\sum\limits_{t\in I} [\uP_t(x)-\uP_t(x_t)].
\]
The regret of the dual algorithm $\AlgD$ is defined similarly for any bounded subset $\cD\subseteq\mathbb{R}_+$:
\[\textstyle
\RD_I(\cD):=\sup_{\vlambda\in\cD}\sum_{t\in I} [\uD_t(\vlambda)-\uD_t(\vlambda_t)].
\]
For ease of notation we write $\RP_T(\cX)$ and $\RD_T(\cD)$ when $I=\range{T}$, instead of $\RP_{\range{T}}(\cX)$ and $\RD_{\range{T}}(\cD)$.

The interaction of $\AlgP$ and $\AlgD$ with the environment is reported in \Cref{alg:alg1}. Note that the feedback of $\AlgP$ is forced to be bandit by the fact that we do not have counterfactual information of $f_t$ and $\vg_t$, however $\AlgD$ receives full feedback by design.

\begin{figure*}[!t]
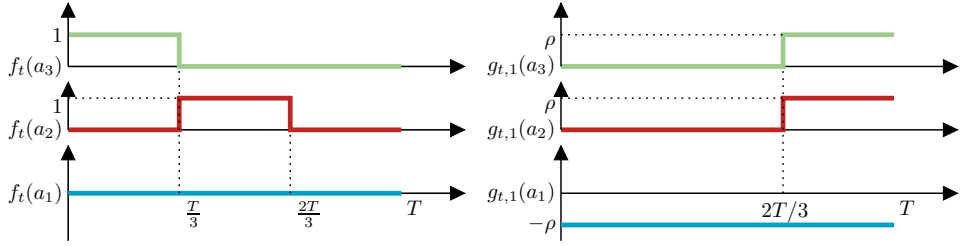

	\begin{subfigure}[t]{.45\textwidth}
		\centering\scalebox{.8}{\tikzset{every picture/.style={line width=0.75pt}} %set default line width to 0.75pt        

\input{colors}

\begin{tikzpicture}[x=0.75pt,y=0.75pt,yscale=-1,xscale=1]
	
	%Straight Lines [id:da6142074245909619] 
	\draw    (130,190) -- (130,133) ;
	\draw [shift={(130,130)}, rotate = 90] [fill={rgb, 255:red, 0; green, 0; blue, 0 }  ][line width=0.08]  [draw opacity=0] (8.93,-4.29) -- (0,0) -- (8.93,4.29) -- cycle    ;
	%Straight Lines [id:da6614164818034962] 
	\draw    (130,160) -- (377,160) ;
	\draw [shift={(380,160)}, rotate = 180] [fill={rgb, 255:red, 0; green, 0; blue, 0 }  ][line width=0.08]  [draw opacity=0] (8.93,-4.29) -- (0,0) -- (8.93,4.29) -- cycle    ;
	%Straight Lines [id:da5292022880084275] 
	\draw    (130,120) -- (377,120) ;
	\draw [shift={(380,120)}, rotate = 180] [fill={rgb, 255:red, 0; green, 0; blue, 0 }  ][line width=0.08]  [draw opacity=0] (8.93,-4.29) -- (0,0) -- (8.93,4.29) -- cycle    ;
	%Straight Lines [id:da2838032140351001] 
	\draw    (130,80) -- (377,80) ;
	\draw [shift={(380,80)}, rotate = 180] [fill={rgb, 255:red, 0; green, 0; blue, 0 }  ][line width=0.08]  [draw opacity=0] (8.93,-4.29) -- (0,0) -- (8.93,4.29) -- cycle    ;
	%Straight Lines [id:da4778835108109314] 
	\draw [color=mgreen ,draw opacity=1 ][line width=2.25]    (130,60) -- (200,60) -- (200,80) -- (340,80) ;
	%Straight Lines [id:da28944708510729344] 
	\draw [color=niceRed  ,draw opacity=1 ][line width=2.25]    (130,120) -- (200,120) -- (200,100) -- (270,100) -- (270,120) -- (340,120) ;
	%Straight Lines [id:da006631978373773162] 
	\draw [color=blueGrotto ,draw opacity=1 ][line width=2.25]    (130,160) -- (200,160) -- (270,160) -- (340,160) ;
	%Straight Lines [id:da4568124485370013] 
	\draw    (130,120) -- (130,93) ;
	\draw [shift={(130,90)}, rotate = 90] [fill={rgb, 255:red, 0; green, 0; blue, 0 }  ][line width=0.08]  [draw opacity=0] (8.93,-4.29) -- (0,0) -- (8.93,4.29) -- cycle    ;
	%Straight Lines [id:da1201906846043117] 
	\draw    (130,80) -- (130,43) ;
	\draw [shift={(130,40)}, rotate = 90] [fill={rgb, 255:red, 0; green, 0; blue, 0 }  ][line width=0.08]  [draw opacity=0] (8.93,-4.29) -- (0,0) -- (8.93,4.29) -- cycle    ;
	%Straight Lines [id:da31417954839902285] 
	\draw  [dash pattern={on 0.84pt off 2.51pt}]  (130,100) -- (200,100) ;
	%Straight Lines [id:da04646253120365018] 
	\draw  [dash pattern={on 0.84pt off 2.51pt}]  (200,160) -- (200,80) ;
	%Straight Lines [id:da9896980535547628] 
	\draw  [dash pattern={on 0.84pt off 2.51pt}]  (270,160) -- (270,120) ;
	
	% Text Node
	\draw (342,163.4) node [anchor=north west][inner sep=0.75pt]    {$T$};
	% Text Node
	\draw (202,163.4) node [anchor=north west][inner sep=0.75pt]    {$\frac{T}{3}$};
	% Text Node
	\draw (272,163.4) node [anchor=north west][inner sep=0.75pt]    {$\frac{2T}{3}$};
	% Text Node
	\draw (128,160) node [anchor=east] [inner sep=0.75pt]    {$f_t( a_{1})$};
	% Text Node
	\draw (128,120) node [anchor=east] [inner sep=0.75pt]    {$f_t( a_{2})$};
	% Text Node
	\draw (128,80) node [anchor=east] [inner sep=0.75pt]    {$f_t( a_{3})$};
	% Text Node
	\draw (128,60) node [anchor=east] [inner sep=0.75pt]    {$1$};
	% Text Node
	\draw (128,105) node [anchor=east] [inner sep=0.75pt]    {$1$};

\end{tikzpicture}}
		%\caption{Support of the valuation's distributions $\mu_k$.}
		% \caption{}
		% \label{fig:instanceonerew}
	\end{subfigure}
	%	\hspace{0.1cm}
	\begin{subfigure}[t]{.45\textwidth}
		\centering\scalebox{.8}{\tikzset{every picture/.style={line width=0.7pt}} %set default line width to 0.75pt        

\input{colors}

\begin{tikzpicture}[x=0.75pt,y=0.75pt,yscale=-1,xscale=1]
	
	%Straight Lines [id:da4228193496675714] 
	\draw    (130,190) -- (130,133) ;
	\draw [shift={(130,130)}, rotate = 90] [fill={rgb, 255:red, 0; green, 0; blue, 0 }  ][line width=0.08]  [draw opacity=0] (8.93,-4.29) -- (0,0) -- (8.93,4.29) -- cycle    ;
	%Straight Lines [id:da8758351492905494] 
	\draw    (130,160) -- (377,160) ;
	\draw [shift={(380,160)}, rotate = 180] [fill={rgb, 255:red, 0; green, 0; blue, 0 }  ][line width=0.08]  [draw opacity=0] (8.93,-4.29) -- (0,0) -- (8.93,4.29) -- cycle    ;
	%Straight Lines [id:da13064750166270445] 
	\draw    (130,120) -- (377,120) ;
	\draw [shift={(380,120)}, rotate = 180] [fill={rgb, 255:red, 0; green, 0; blue, 0 }  ][line width=0.08]  [draw opacity=0] (8.93,-4.29) -- (0,0) -- (8.93,4.29) -- cycle    ;
	%Straight Lines [id:da03898749162876736] 
	\draw    (130,80) -- (377,80) ;
	\draw [shift={(380,80)}, rotate = 180] [fill={rgb, 255:red, 0; green, 0; blue, 0 }  ][line width=0.08]  [draw opacity=0] (8.93,-4.29) -- (0,0) -- (8.93,4.29) -- cycle    ;
	%Straight Lines [id:da8617752734941275] 
	\draw [color=mgreen ,draw opacity=1 ][line width=2.25]    (130,80) -- (270,80) -- (270,60) -- (340,60) ;
	%Straight Lines [id:da9041618341890725] 
	\draw [color=niceRed  ,draw opacity=1 ][line width=2.25]    (130,120) -- (270,120) -- (270,100) -- (340,100) ;
	%Straight Lines [id:da6186108168752826] 
	\draw [color=blueGrotto  ,draw opacity=1 ][line width=2.25]    (130,180) -- (200,180) -- (270,180) -- (340,180) ;
	%Straight Lines [id:da9950776978608451] 
	\draw    (130,120) -- (130,93) ;
	\draw [shift={(130,90)}, rotate = 90] [fill={rgb, 255:red, 0; green, 0; blue, 0 }  ][line width=0.08]  [draw opacity=0] (8.93,-4.29) -- (0,0) -- (8.93,4.29) -- cycle    ;
	%Straight Lines [id:da9635079991672844] 
	\draw    (130,80) -- (130,43) ;
	\draw [shift={(130,40)}, rotate = 90] [fill={rgb, 255:red, 0; green, 0; blue, 0 }  ][line width=0.08]  [draw opacity=0] (8.93,-4.29) -- (0,0) -- (8.93,4.29) -- cycle    ;
	%Straight Lines [id:da5764556704760286] 
	\draw  [dash pattern={on 0.84pt off 2.51pt}]  (130,60) -- (270,60) ;
	%Straight Lines [id:da8088833217473335] 
	\draw  [dash pattern={on 0.84pt off 2.51pt}]  (130,100) -- (270,100) ;
	%Straight Lines [id:da4680958916665392] 
	\draw  [dash pattern={on 0.84pt off 2.51pt}]  (270,160) -- (270,80) ;
	
	% Text Node
	\draw (342,163.4) node [anchor=north west][inner sep=0.75pt]    {$T$};
	% Text Node
	\draw (255,163.4) node [anchor=north west][inner sep=0.75pt]    {$2T/3$};
	% Text Node
	\draw (128,160) node [anchor=east] [inner sep=0.75pt]    {$g_{t,1}( a_{1})$};
	% Text Node
	\draw (128,120) node [anchor=east] [inner sep=0.75pt]    {$g_{t,1}( a_{2})$};
	% Text Node
	\draw (128,80) node [anchor=east] [inner sep=0.75pt]    {$g_{t,1}( a_{3})$};
	% Text Node
	\draw (128,180) node [anchor=east] [inner sep=0.75pt]    {$-\rho $};
	% Text Node
	\draw (128,105) node [anchor=east] [inner sep=0.75pt]    {$\rho $};
	% Text Node
	\draw (128,65) node [anchor=east] [inner sep=0.75pt]    {$\rho $};

\end{tikzpicture}}
		%\caption{Expected \gft~of prices $(p,p+\delta)$.}
		% \caption{}
		% \label{fig:instanceonecosts}
	\end{subfigure}
	\caption{Reward and costs of each arm of the instance employed in \Cref{ex:instanceone}.}
	\label{fig:instanceone}
\end{figure*}

\begin{remark}[The Challenges of the Adversarial Setting]
In the stochastic setting, it is not required adaptive regret minimization, see \eg \cite{slivkins2023contextual}, as it is possible to analyze directly the expected zero-sum game between $\AlgP$ and $\AlgD$. However, in the adversarial setting, the algorithms $\AlgP$ and $\AlgD$ face a different zero-sum game at each time $t$. Indeed, since $f_t$ and $g_t$ are adversarial, the zero-sum game with payoffs $\cL_{f_t,\vg_t}(\cdot,\cdot)$ is only seen at time $t$. This is in contrast to what happens in the stochastic setting in which the zero-sum game $\cL_{\bar f,\bar \vg}(\cdot,\cdot)$ at each time $t$ is the same for all time $t$.
\end{remark}

\subsection{No-Regret is Not Enough!}\label{sec:example no regret not enough}

Typically, Lagrangian frameworks for constrained bandit problems are solved by instantiating $\AlgP$ and $\AlgD$ with two regret minimizers, which are algorithms guaranteeing $\RP_{T}(\cX),\RD_{T}(\cD)=o(T)$, respectively \cite{immorlica2022jacm,castiglioni2022online}. The dual regret minimizer is usually instantiated with $\cD\defeq [0,M]^m$, for some constant $M>0$. Ensuring that $\cD$ is bounded is crucial to control the magnitude of primal utilities $\uP_t(\cdot)$, whose scale influences the magnitude of the primal regret.
In the following example, we show that the simple no-regret property alone of $\AlgP$ and $\AlgD$ is not sufficient in our setting.  
\begin{example}\label{ex:instanceone}
	We have one constraint, \ie $m=1$ and the set $\cX=\{a_1,a_2,a_3\}$ is a discrete set of $3$ actions. 
	The rewards of $a_1$ is always $0$, \ie $f_t(a_1)=0$ for all $t\in\range{T}$, while its cost is always $-\rho$, \ie $g_{t,1}(a_1)=-\rho$ for all $t\in t$.
	The rewards for $a_2$ and $a_3$ are defined as follows: for $t\in\range{T/3}$ we have $f_t(a_2)=0$ while $f_t(a_3)=1$. On the other hand, for $t\in\range{T/3, 2T/3}$ we have $f_t(a_2)=1$ while $f_t(a_3)=0$. Finally $f_t(a_2)=f_t(a_3)=0$ for all $t\in\range{2T/3,T}$.
	The costs for $a_2$ and $a_3$ are defined as follows: for $t\in\range{2T/3}$ we have $g_{t,1}(a_2)=g_{t,1}(a_3)=0$, while $g_{t,1}(a_2)=g_{t,1}(a_3)=1$ for all $t\in\range{2T/3,T}$.	
	The instance is depicted in \Cref{fig:instanceone}.
\end{example}

\begin{restatable}{proposition}{propRegretNotEnought}\label{prop:RegretNotEnought}
	Consider the instance of \Cref{ex:instanceone}. Even if $\AlgP$ and $\AlgD$ suffer regret less than or equal then zero, the primal-dual framework fails to achieve sublinear constraint violations.
\end{restatable}

Intuitively, the reason for which a standard primal-dual framework fails in \Cref{ex:instanceone} is that the primal regret minimizer can accumulate enough negative regret in the first two phases to ``absorb'' large regret suffered in the third phase. This ``laziness'' of $\AlgP$ allows it to play actions in the last phase for which it incurs  linear violations of the constraint. For more details see the proof of \Cref{prop:RegretNotEnought} in \Cref{app:A}.
One could solve the problem employing the \emph{recovery technique} proposed in \cite{castiglioni2022unifying}, which prescribes to minimize the violations at a prescribed time. However, selecting the right time to start the recovery phase crucially requires knowledge of the Slater's parameter, which is not available in our setting. 
The only approach which does not require knowledge of Slater's parameter is the one proposed in \cite{castiglioni2023online} for the case of \emph{return-on-investment} constraints, whose core idea we describe in the next section.

\subsection{No-Adaptive Regret}
The reason why generic regret minimizes fail to give satisfactory result on the instance described in \Cref{ex:instanceone} is that they fail to adapt to the changing environment, even if the regret of the primal is zero on the entire horizon $\range{T}$, it fails to ``adapt'' in the final rounds $\range{2T/3, T}$. Indeed, in these last rounds, if the primal algorithm's objective is guaranteeing sublinear regret over $\range{T}$, it is not required to updated its decision, since it accumulated large negative regret of $-2T/3$ regret in the initial rounds $\range{2T/3}$. Therefore, standard no-regret guarantees are not enough. 

A stronger requirement for the primal and dual algorithm is being \emph{weakly adaptive} \citep{hazan2007adaptive}, that is, guaranteeing that in high probability
\(
\sup_{I=\range{t_1,t_2}} R^{\term{P}, \term{D}}_I=o(T).
\)
Intuitively, this requirement would force $\AlgP$ to change its action during the last phase of \Cref{ex:instanceone}. 
This idea was first proposed in \cite{castiglioni2023online} for the specific case of a learner with one budget and one return-on-investments constraints. 
In the following section, we show how such approach can be extended to the case of general constraints.

\section{Self-Bounding Lemma}\label{sec:selfbound}

%In particular, in the adversarial setting, we cannot show that if the primal algorithm employs a No-regret strategy \ma{---} then the violations are small.

One crucial difference with the previous literature is that the feasibility parameter is not known a priori, and thus we cannot directly bound the range of the Lagrange multipliers as in \bwk. 
%
%Similarly to problems on \bwk we need the following property of \emph{self-boundedness} of the Lagrange multipliers to hold.
%
At a high level we want that, regardless of the choices of $f_t$ and $\vg_t$, the $\ell_1$ norm of the Lagrange multipliers is bounded by a quantity that depends on the (unknown) parameters of the instance. However, for this to hold we need that the primal algorithm $\AlgP$ is (almost) scale free, \ie that its regret scale quadratically in the unknown range of its reward function.\footnote{Usually we say that an algorithm is scale-free \citep{orabona2018scale} if its regret scales linearly in the (unknown) range of its rewards, \ie $1$-scale-free with our definition.}
Formally:
\begin{definition}
	For any $c\ge 1$, we say that $\AlgP$ is a $c$-scale-free and weakly-adaptive regret minimizer if, for any subset of rounds  $I=\range{t_1,t_2}\subseteq \range{T}$, with probability at least $1-\delta$ it holds that
	\[
	\RP_I(\cX)\le L^c\cdot \overline{\RP}_{T,\delta}(\cX),
	\]
	where the maximum module of the primal utilities is $\sup_{t\in\range{T},x\in \cX}|\uP_t(x)|\eqqcolon L$, and $\overline{\RP}_{T,\delta}(\cX)$ depends only on $T$, $\delta$ and $\cX$, and is non-decreasing in the length of the time horizon $T$.
\end{definition} 

Now, we show that \emph{online gradient descent} ($\OGD$) \citep{zinkevich2003online} with a carefully defined learning rate yields the required self-bounding property both in the stochastic and adversarial setting.
\begin{restatable}[Self-bounding lemma]{lemma}{lemmaselfbounded}\label{thm1}
	Let $\eta_\OGD:=\left({800\cdot m \cdot  \max \left\{    {\overline{\RP}_{T,\delta}(\cX)}, E_{T,\delta}\right\} }\right)^{-1} $, then if $\AlgD$ is $\OGD$ on the set $\cD=\mathbb{R}^m_{\ge 0}$, and the primal algorithm $\AlgP$ is $2$-scale-free and has a high-probability weakly adaptive regret bound $\overline{\RP}_{T,\delta}(\cX)$, then with probability at least $1-\delta$: 
	\[\textstyle
	\max_{t\in\range{T}}\|\vlambda_t\|_1\le\frac{13m}{\rho},
	\]
	where $\rho=\rhoadv$ or $\rho=\rhostoc$ depending on the setting and $E_{T,\delta}\coloneqq\sqrt{16T\log\left(\nicefrac{2T}{\delta}\right)}$.
\end{restatable}

We remark that the self-bounding lemma shows that, if we take $\OGD$ with a carefully defined learning rate $\eta_\OGD=\widetilde O((m\max\{\overline{\RP}_{T,\delta}(\cX),\sqrt{T}\})^{-1})$ as $\AlgP$, then  the $\ell_1$-norm of the variables $\vlambda_t$ is automatically bounded by the reciprocal of the feasibility parameter, even if the feasibility parameter is unknown to the learner. This is the central result that allows us to build algorithms that work without knowing Slater's parameter. We observe that:

\begin{remark}Even in the simplest instances of bandit problems one has $\overline{\RP}_{T,\delta}(\cX)=\widetilde \Omega(\sqrt{T})$ and, therefore, we can assume that $\eta_\OGD=\widetilde O\left(({m\overline{\RP}_{T,\delta}(\cX)})^{-1}\right)$.
\end{remark}

\begin{remark}
We will work with $2$-scale-free algorithms, which suffice to obtain the desired guarantees for our framework. We observe that scale-free algorithms would yield a tighter bound of $1/\rho$ in the \Cref{th:advregret,th:stochregret} and a simpler analysis of \Cref{thm1}. However, scale-free algorithm are much more difficult to find and this would limit the extent to which our framework can be applied. On the other hand, $2$-scale-free algorithm seems to be more abundant (see, \eg \Cref{sec:applications}). Indeed, as we show in \Cref{sec:applications}, it is usually the case that setting the learning rate independent on the scale of the rewards provides $2$-scale-freeness. We leave such characterization to future research.
\end{remark}

%\section{Reward} \label{sec:rewards}

\section{General Guarantees}\label{sec:general}

First, we exploit \Cref{thm1} to bound the total violations of the framework.

\begin{restatable}{theorem}{theoremviolations}\label{th:violations}
	Let $\AlgD$ be $\OGD$ with learning rate $\eta$ as in \Cref{thm1}, and let $\AlgP$ any $2$-scale-free algorithm with no-adaptive regret.
	Then, with probability at least $1-\delta$, it holds that \(V_T=\widetilde O\left(\frac{m^2}{\rho}\overline{\RP}_{T,\delta}(\cX)\right),\) where $\rho=\rhoadv$ in the adversarial setting and $\rho=\rhostoc$ in the stochastic.
\end{restatable}

Moreover, the proof of \Cref{th:violations} can be easily adapted to show that the violations of any constraint $i\in\range{m}$ is bounded on any interval $\range{t}$ with $t\in\range{T}$.

Now, we prove that the framework, with high probability, yields optimal guarantees in both stochastic and adversarial settings. We start with the adversarial setting, for which the following result holds.

\begin{restatable}{theorem}{advregret}\label{th:advregret}
	If $\AlgD$ is $\OGD$ with learning rate $\eta_\OGD$ and domain $\cD:=\mathbb{R}^m_{\ge 0}$, and $\AlgP$ is $2$-scale-free, then, in the adversarial setting, with high probability:
	\[
	\Rew\ge \frac{\rhoadv}{1+\rhoadv}\optadv- \widetilde O\left(\left(\frac m\rhoadv\right)^2\overline{\RP}_{T,\delta}(\cX)\right).
	\]
\end{restatable}

On the other hand, for the stochastic setting we can prove the following result:

\begin{restatable}{theorem}{stocregret}\label{th:stochregret}
	If $\AlgD$ is $\OGD$ with learning rate $\eta_\OGD$ and domain $\cD:=\mathbb{R}^m_{\ge 0}$, and $\AlgP$ is $2$-scale-free, then in the stochastic setting, in high probability:
%	\[
%	\Rew\ge  \optstoc-O\left(L^2\overline{\RP}_{T,\delta}(\cX)+M E_{T,\delta}+\eta_\OGD mT\right).
%	\]
	\[
	\Rew\ge \optstoc-\widetilde O\left(\left(\frac{m}{\rhostoc}\right)^2\overline{\RP}_{T,\delta}(\cX)\right).
	\]
\end{restatable}

\begin{remark}
	Any algorithm with vanishing constraints violations can be employed to handle also \bwk constraints. In such setting, the learner has resource-consumption constraints with \emph{hard stopping} (\ie once the budget for a resource is fully depleted the learner must play the void action until the end of time horizon). 
	This does not yield any fundamental complication for our framework. Indeed, we could introduce an initial phase of $o(T)$  rounds in which the algorithm collects the extra budget needed to cover potential violations, before starting the primal-dual procedure.  
	% Notice that any algorithm with bounded violations can be used in resource consumption only, such as \bwk. These problems are characterize by \emph{hard stopping}, \ie the algorithm must resort to play a null action as soon as some resource is fully consumed. On the other hand, our model this is not contemplated.	
	% However this is not a fundamental difference, indeed, if we know a bound on our violations, we could spend initial time to collect the needed extra budget before committing to our primal-dual algorithm.
\end{remark}
\section{Applications}\label{sec:applications}

In this section, we show how our framework can be instantiated to handle scenarios such as bandits with general constraints, as well as contextual bandits with constraints (\ie \cbwlc).
Thanks to the modularity of the results derived in the previous sections, we only need to provide an algorithm $\AlgP$ which is $2$-scale-free and weakly adaptive for a desired action space $\cX$ and rewards $\uP_t$. 

\subsection{Bandits with General Constraints}

In this setting, the action space is $\cX=\range{K}$. \cite{castiglioni2023online} showed that the $\EXPSIX$ algorithm introduced by \cite{neu2015explore} can be used as $\AlgP$, since it guarantees sublinear weakly adaptive regret in high probability, and it is $2$-scale-free. 
\begin{theorem}[Theorem 8.1 of \cite{castiglioni2023online}]\label{thm:adaptive regret primal short}
	$\EXPSIX$ instantiated with suitable parameters guarantees that, with probability at least $1-\delta$ that
	\(
	\sup_{I=\range{t_1,t_2}}  \RP_{I}(\cX)= O\mleft(\sqrt{KT}\log\mleft({K T}{\delta^{-1}}\mright)\mright).
	\)
\end{theorem}

Thus, by applying \cref{th:violations} on the violations, and \cref{th:advregret} and \cref{th:stochregret} on the adversarial and stochastic reward guarantees respectively, we get the following result:

\begin{corollary}
	Consider a multi armed bandit problem with constraints. There exists an algorithm that w.h.p.~guarantees, in the adversarial setting, violations at most $\tilde O\left(\frac{m^{2}}{\rhoadv}\sqrt{KT}\right)$ and
	\(  \Rew\ge \frac\rhoadv{1+\rhoadv}\optadv -\tilde O\left(\frac{m^{2}}{\rhoadv^2} \sqrt {KT} \right),   \)
	while,in the stochastic setting, it guarantees violations at most $\tilde O\left(\frac{m^{2}}{\rhostoc}\sqrt{KT}\right)$ and reward at least
	\(      \Rew\ge \optstoc - \tilde O\left(\frac{m^{2}}{\rhostoc^2} \sqrt{K T} \right).
	\)
	
\end{corollary}

\subsection{Contextual Bandits with Constraints}\label{sec:Lagrang}

Following \cite{SlivkinsSF23}, we apply our general framework to contextual bandits with regression oracles. %The interaction between the agent and the environment is outlined in the following.
In this setting, the decision maker observes a context $z_t\in\cZ$ from some context set $\cZ$, where $z_t$ is possibly chosen by an adversary. Then, the decision maker picks its decision $a_t$ from an action set $\cA$. Then, the reward is computed as a function of the context and the action, \ie $f_t:\cZ\times\cA\to[0,1]$, and similarly for the constraints $\vg_t:\cZ\times\cA\to[-1,1]^m$. At each $t$, $f_t$ and $\vg_t$ are drawn from some distribution. More precisely, there exist a class $\cF$ of functions and $\bar f,\bar g_i\in \cF$ such that for all $(z,a)\in\cZ\times\cA$ it holds that  $\E[f_t(z,a)|z,a]=\bar f(z,a)$ and $\E[g_{t,i}(z,a)|z,a]=\bar g_i(z,a)$ for $i\in\range{m}$.

We slightly modify the primal-dual algorithm to handle contexts. In particular, $\AlgP$ gets to observe a context $z_t$ before deciding their action. Formally, we can use the machinery introduced in \Cref{sec:prelim} by taking $\cX$ as the set of deterministic policies $\Pi\coloneqq\{\pi:\cZ\to\cA\}$. Then, $\uP_t(\pi)=f_t(z_t, \pi(z_t))-\langle \vlambda_t, \vg_t(z_t, \pi(z_t)) \rangle$, and the action $a_t$ is computed through $\pi_t$ returned by the primal algorithm.
%
% While this choice makes the contextual framework an application of the general framework introduced in \Cref{sec:prelim}, in practice, it is simpler to think about $a_t$ as directly returned by the primal algorithm after observing the context $z_t$.
%
Although this choice transforms the contextual framework into an application of the framework introduced in \Cref{sec:prelim}, in practical terms, it is simpler to think of $a_t$ as the direct output of $\AlgP$ upon observing the context $z_t$. The extended primal-dual framework is sketched in \Cref{alg:alg2}.

\begin{figure}[!t]
	\begin{minipage}[t]{0.44\textwidth}
		\begin{algorithm}[H]
			\caption{Primal-Dual Algorithm\\ for Contextual Bandits}
			\label{alg:alg2}
			\begin{algorithmic}[1]
				\STATE {\bfseries Input:} $\AlgP$ and $\AlgD$.
				\FOR{$t = 1, 2, \ldots , T$}
				\STATE Observe context $z_t$ %and feed it to $\AlgP$
				\STATE {\bfseries Dual decision:} $\vlambda_t \gets \AlgD$
				\STATE{\bfseries Primal decision:} 
				\begin{ALC@g}
					%\STATE Feed $z_t$ and $\vlambda_t$ to $\AlgP$
					\STATE $a_t \gets \AlgP(z_t,\vlambda_t)$
				\end{ALC@g}
				\STATE {\bfseries Observe:} $f_t(z_t,a_t)$ and $\vg_t(z_t,a_t)$ %\LComment{bandit feedback}
				\STATE {\bfseries Primal update:} feed $\uP_t(a_t)$ to $\AlgP$, where
				\begin{ALC@g}
					\STATE \hspace{-0.1cm}$\uP_t(a_t)\hspace{-0.05cm} = \hspace{-0.05cm} f_t(z_t,a_t) \hspace{-0.05cm}  - \hspace{-0.05cm}  \langle \vlambda_t, \vg_t(z_t,a_t)\rangle$
					%\STATE Feed $\uP_t(a_t)$ to $\AlgP$
				\end{ALC@g}
				\STATE {\bfseries Dual update:} feed $\uD_t$ to $\AlgD$, where
				\begin{ALC@g}
					\item $\uD_t(\vlambda)- f_t(z_t,a_t)+\langle\vlambda,\vc_t(z_t,a_t)\rangle$
				\end{ALC@g}
				\ENDFOR
				\vspace{0.175cm}
			\end{algorithmic}
		\end{algorithm}
	\end{minipage}
	\hfill
	\begin{minipage}[t]{0.55\textwidth}
		\begin{algorithm}[H]
			\caption{Primal Algorithm for Contextual Bandits}
			\label{alg:algPrimalContex}
			\begin{algorithmic}[1]
				\STATE {\bfseries Input:} Learning rate $\eta_\term{P}$
				%		\STATE {\bfseries Input:} Online regression oracles: $(\cO_f,\cO_1,\ldots, \cO_m)$
				%		\FOR{$t = 1, 2, \ldots , T$}
				\STATE {\bfseries Get regressors from online regression oracles:} 
				\begin{ALC@g}
					\STATE $\hat f_t\gets \cO_f$, and $\hat g_{t,i}\gets \cO_i$ for all $i\in\range{m}$
				\end{ALC@g}
				\STATE Observe context $z_t$ and dual variable $\vlambda_t$
				\STATE For all $a\in\cA$ compute $\hat\cL_t(a):=\cL_{\hat f_t, \hat \vg_{t}}((z_t,a), \vlambda_t)$
				\STATE Compute $\xi_t\in\Delta(\cA)$ as:
				\vspace{-2mm}
				\[
				\xi_t(a)=\left(\mu_t+\eta_\term{P}\left(\max_{a'}\hat \cL_t(a')-\hat\cL_t(a)\right)\right)^{-1}
				\]\LComment{$\mu_t$ is such that $\xi_t\in\Delta(\cA)$}
				\STATE Sample $a_t\sim \xi_t$ and return it.
				\STATE {\bfseries Update online regression oracles:} 
				\begin{ALC@g}
					\STATE Feed $(z_t,a_t,f_t(z_t,a_t))$ to $\cO_f$
					\STATE Feed $(z_t,a_t,g_{t,i}(z_t,a_t))$ to $\cO_i$  $\forall i\in\range{m}$
				\end{ALC@g}
				%		\ENDFOR
			\end{algorithmic}
		\end{algorithm}
	\end{minipage}
\end{figure}

We assume to have $m+1$ online regression oracles $(\cO_f,\cO_1,\ldots, \cO_m)$ for the functions $\bar f$ and $\bar g_1,\ldots, \bar g_m$, respectively. The regression oracle $\cO_f$ produces, at each $t$, a regressor $\hat f_t\in\cF$ that tries to approximate the \emph{true} regressor $\bar f$. Then, the oracle is feed with a new data point, comprised of a context $z_t\in\cZ$ and an action $a_t\in\cA$, and the performance of the regressor is evaluated on the basis of its prediction for the tuple $(z_t,a_t)$. 
The online regression oracle $\cO_f$ is updated with the labeled data point $(z_t, a_t, f_t(z_t,a_t))$.
Overall, its performance is measured by its cumulative $\ell_2$-error:
\[\textstyle
\err(\cO_f)\coloneqq\sum_{t\in\range{T}}\left(\hat f_t(z_t, a_t)-\bar f(z_t,a_t)\right)^2.
\]
Each online regression oracle $(\cO_i)_{i\in\range{m}}$ works analogously, and its performance is measured by
\(
\err(\cO_i)\coloneqq\sum_{t\in\range{T}}\left(\hat g_t(z_t, a_t)-\bar g(z_t,a_t)\right)^2.
\)

By combining the online regression oracles $\cO_f$ and $\{\cO_i\}_{i\in\range{m}}$ we can build an online regression oracle $\cO_\cL$ for the Lagrangian which outputs  regressors $\hat \cL_t:\cZ\times\cA\to\mathbb{R}$ defined as:
\begin{align*}
\hat \cL_t(z,a) &= \cL_{\hat f_t, \hat \vg_t}((z, a), \vlambda_t)=\hat f_t((z,a)) - \langle\vlambda_t, \hat \vg_t(z,a)\rangle,
\end{align*}
while we define $\bar \cL(z,a)\coloneqq\cL_{\bar f, \bar \vg}((z,a), \vlambda_t)$.
The $\ell_2$-error of $\cO_\cL$ can be bounded via the following extension of \cite[Theorem~16]{slivkins2023contextual}.

\begin{restatable}{lemma}{lemmaerrL}\label{lem:La}
	The error of $\cO_\cL$ can be bounded as
	\[\textstyle
	\err(\cO_\cL)\le 2\err(\cO_f)+2\left(\sup_{t\in\range{T}}\|\vlambda_t\|_1\right)^2\hspace{-2mm}\sum_{i\in \range{m}}\err(\cO_i).
	\]
\end{restatable}

The fundamental idea of \cite{foster2020beyond} is to reduce (unconstrained) contextual bandit problems to online linear regression. Recently, this ideas was extended in \cite{SlivkinsSF23,han2023optimal} in order to design a primal algorithm $\AlgP$ capable of handling stochastic contextual bandits with constraints (see \Cref{alg:algPrimalContex}).

To apply \Cref{alg:algPrimalContex} to our framework we need to find an algorithm $\AlgP$ which is $2$-scale-free and weakly adaptive with high probability.
We extend the result \citep{foster2020beyond} to prove that their reduction actually satisfies the required guarantees.

\begin{restatable}{lemma}{theoremadaptiveContextual}\label{th:adaptivecontextual}
	Assume that $\max\{\err(\cO_f),\err(\cO_i)\}\le \overline{\err}$. Then, we have that \Cref{alg:algPrimalContex} with $\eta_\term{P}\coloneqq\sqrt{KT}$ guarantees that
	\(
	\sup_{I=\range{t_1,t_2}}\RP_I(\Pi)=\tilde O\left( m\cdot \overline{\err}\cdot L^2 \cdot \sqrt{KT}\right)
	\)
	with high probability, where $L:=\sup_{t\in\range{T}, \pi\in\Pi} |\uP_t(\pi)|$.
\end{restatable}

Equipped with a $2$-scale free algorithm that suffers no adaptive regret with high probability, we can combine $\AlgP$ with the results of \cref{th:violations,th:advregret,th:stochregret} to prove the first optimal guarantees for \cbwlc with adversarial contexts.
%
%\ac{non possiamo evitare there exists?}
%
\begin{corollary}
		Consider a functional class $\cF$ and an online regression oracle that guarantees  $\ell_2$-error $\overline{\err}$. There exists an algorithm that w.h.p.~guarantees violations at most $\tilde O\left(\frac{m^{3}}{\rhoadv}\overline{\err}\sqrt{KT}\right)$ and reward at least
	\( \Rew\ge \frac\rhoadv{1+\rhoadv}\optadv -\tilde O\left(\overline{\err}\frac{m^3}{\rhoadv^2}\sqrt{KT} \right)     \)
	in the adversarial setting, while it guarantees violations at most $\tilde O\left(\frac{m^{3}}{\rhostoc}\overline{\err} \sqrt{KT}\right)$ and reward at least
	\(    \Rew\ge   \optstoc -\tilde O\left(\overline{\err}\frac{m^3}{\rhostoc^2}\sqrt{KT} \right)  
	\)
	in the stochastic setting.
\end{corollary}

\cite{foster2020beyond} includes many examples of functional classes $\cF$ that have good online regression oracles, meaning that their error is subpolynomial in the time horizon $T$. We report here some notable mentions for completeness. 

If $\cF$ is a finite set of functions we have that $\overline\err=O(\log|\cF|)$, which comes from using as regression oracles the Vovk forecaster \cite{vovk1995game}.
Another important examples is the case in which $\cF$ is the class of linear functions, \ie $\cF=\{h(z,a)=\langle z_a,\theta\rangle: \theta\in\mathbb{R}^d, \|\theta\|_2\le 1\}$, \ie each actions $a$ is associated with a known feature vector $z_a\in\mathbb{R}^d$ which generates the reward/costs trough a unknown parameter $\theta$ that characterize the linear function. Here, there exists a online regression oracle which provides $\ell_2$-error $\overline\err=O(d\log(T/d))$ \cite{azoury2001relative}.

\bibliography{bib-abbrv,refs}
\bibliographystyle{plainnat}
\clearpage
\appendix
\onecolumn

%\section*{Appendix}

\section{Omitted Proofs from \Cref{sec:best of both worlds} and \Cref{sec:Lagrangian}}\label{app:A}

\begin{figure*}[!th]
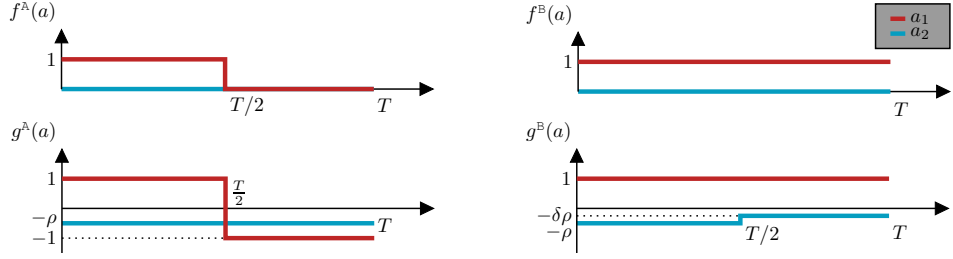

	\begin{subfigure}[t]{1\linewidth}
		\centering\scalebox{.75}{\tikzset{every picture/.style={line width=0.75pt}} %set default line width to 0.75pt        

\input{colors}

\begin{tikzpicture}[x=0.75pt,y=0.75pt,yscale=-1,xscale=1]
	
	%Straight Lines [id:da43527300291242477] 
	\draw    (130,120) -- (377,120) ;
	\draw [shift={(380,120)}, rotate = 180] [fill={rgb, 255:red, 0; green, 0; blue, 0 }  ][line width=0.08]  [draw opacity=0] (8.93,-4.29) -- (0,0) -- (8.93,4.29) -- cycle    ;
	%Straight Lines [id:da3159070154747676] 
	\draw [color=blueGrotto ,draw opacity=1 ][line width=2.25]    (130,120) -- (200,120) -- (270,120) -- (340,120) ;
	%Straight Lines [id:da826154188292145] 
	\draw    (130,120) -- (130,83) ;
	\draw [shift={(130,80)}, rotate = 90] [fill={rgb, 255:red, 0; green, 0; blue, 0 }  ][line width=0.08]  [draw opacity=0] (8.93,-4.29) -- (0,0) -- (8.93,4.29) -- cycle    ;
	%Straight Lines [id:da1316547932136849] 
	\draw [color=niceRed  ,draw opacity=1 ][line width=2.25]    (130,100) -- (240,100) -- (240,120) -- (340,120) ;
	
	% Text Node
	\draw (341,125) node [anchor=north west][inner sep=0.75pt]    {$T$};
	% Text Node
	\draw (242,123.4) node [anchor=north west][inner sep=0.75pt]    {$T/2$};
	% Text Node
	\draw (128,76.6) node [anchor=south east] [inner sep=0.75pt]    {$f^{\insta}( a)$};
	% Text Node
	\draw (128,100) node [anchor=east] [inner sep=0.75pt]    {$1$};

\end{tikzpicture}}
		\centering\scalebox{.75}{\tikzset{every picture/.style={line width=0.75pt}} %set default line width to 0.75pt        

\input{colors}

\begin{tikzpicture}[x=0.75pt,y=0.75pt,yscale=-1,xscale=1]
	
	%Straight Lines [id:da18526551344402908] 
	\draw    (150,140) -- (397,140) ;
	\draw [shift={(400,140)}, rotate = 180] [fill={rgb, 255:red, 0; green, 0; blue, 0 }  ][line width=0.08]  [draw opacity=0] (8.93,-4.29) -- (0,0) -- (8.93,4.29) -- cycle    ;
	%Straight Lines [id:da10572744400761658] 
	\draw [color=blueGrotto  ,draw opacity=1 ][line width=2.25]    (150,140) -- (220,140) -- (290,140) -- (360,140) ;
	%Straight Lines [id:da8926220913408558] 
	\draw    (150,140) -- (150,103) ;
	\draw [shift={(150,100)}, rotate = 90] [fill={rgb, 255:red, 0; green, 0; blue, 0 }  ][line width=0.08]  [draw opacity=0] (8.93,-4.29) -- (0,0) -- (8.93,4.29) -- cycle    ;
	%Straight Lines [id:da9614121958392616] 
	\draw [color=niceRed  ,draw opacity=1 ][line width=2.25]    (150,120) -- (360,120) ;
	%Shape: Rectangle [id:dp5094557079570721] 
	\draw  [fill={rgb, 255:red, 155; green, 155; blue, 155 }  ,fill opacity=0.38 ] (350,81) -- (400,81) -- (400,111) -- (350,111) -- cycle ;
	%Straight Lines [id:da24089861553782543] 
	\draw [color=niceRed  ,draw opacity=1 ][line width=2.25]    (360,91) -- (370,91) ;
	%Straight Lines [id:da028735129641686452] 
	\draw [color=blueGrotto  ,draw opacity=1 ][line width=2.25]    (360,101) -- (370,101) ;
	
	% Text Node
	\draw (361,146) node [anchor=north west][inner sep=0.75pt]    {$T$};
	% Text Node
	\draw (148,96.6) node [anchor=south east] [inner sep=0.75pt]    {$f^{\instb}( a)$};
	% Text Node
	\draw (148,120) node [anchor=east] [inner sep=0.75pt]    {$1$};
	% Text Node
	\draw (372,101) node [anchor=west] [inner sep=0.75pt]    {$a_{2}$};
	% Text Node
	\draw (372,91) node [anchor=west] [inner sep=0.75pt]    {$a_{1}$};

\end{tikzpicture}}
		%\caption{Support of the valuation's distributions $\mu_k$.}
		% \caption{}
		% \label{fig:instancetworew}
	\end{subfigure}
	\begin{subfigure}[t]{1\textwidth}
		\centering\scalebox{.75}{\tikzset{every picture/.style={line width=0.75pt}} %set default line width to 0.75pt        

\input{colors}

\begin{tikzpicture}[x=0.75pt,y=0.75pt,yscale=-1,xscale=1]
	
	%Straight Lines [id:da40380484725188137] 
	\draw    (150,140) -- (397,140) ;
	\draw [shift={(400,140)}, rotate = 180] [fill={rgb, 255:red, 0; green, 0; blue, 0 }  ][line width=0.08]  [draw opacity=0] (8.93,-4.29) -- (0,0) -- (8.93,4.29) -- cycle    ;
	%Straight Lines [id:da9335761737481414] 
	\draw [color=blueGrotto  ,draw opacity=1 ][line width=2.25]    (150,150) -- (220,150) -- (290,150) -- (360,150) ;
	%Straight Lines [id:da3435445867141127] 
	\draw    (150,170) -- (150,103) ;
	\draw [shift={(150,100)}, rotate = 90] [fill={rgb, 255:red, 0; green, 0; blue, 0 }  ][line width=0.08]  [draw opacity=0] (8.93,-4.29) -- (0,0) -- (8.93,4.29) -- cycle    ;
	%Straight Lines [id:da919481196642796] 
	\draw [color=niceRed  ,draw opacity=1 ][line width=2.25]    (150,120) -- (260,120) -- (260,160) -- (360,160) ;
	%Straight Lines [id:da34056424151457576] 
	\draw  [dash pattern={on 0.84pt off 2.51pt}]  (150,160) -- (260,160) ;
	
	% Text Node
	\draw (361,145) node [anchor=north west][inner sep=0.75pt]    {$T$};
	% Text Node
	\draw (262,120) node [anchor=north west][inner sep=0.75pt]    {$\frac{T}{2}$};
	% Text Node
	\draw (148,96.6) node [anchor=south east] [inner sep=0.75pt]    {$g^{\insta}(a)$};
	% Text Node
	\draw (148,120) node [anchor=east] [inner sep=0.75pt]    {$1$};
	% Text Node
	\draw (148,160) node [anchor=east] [inner sep=0.75pt]    {$-1$};
	% Text Node
	\draw (148,147) node [anchor=east] [inner sep=0.75pt]    {$-\rho $};

\end{tikzpicture}}
		\centering\scalebox{.75}{\tikzset{every picture/.style={line width=0.75pt}} %set default line width to 0.75pt        

\input{colors}

\begin{tikzpicture}[x=0.75pt,y=0.75pt,yscale=-1,xscale=1]
	
	%Straight Lines [id:da49899153698011656] 
	\draw    (150,140) -- (397,140) ;
	\draw [shift={(400,140)}, rotate = 180] [fill={rgb, 255:red, 0; green, 0; blue, 0 }  ][line width=0.08]  [draw opacity=0] (8.93,-4.29) -- (0,0) -- (8.93,4.29) -- cycle    ;
	%Straight Lines [id:da785340819236144] 
	\draw [color=blueGrotto  ,draw opacity=1 ][line width=2.25]    (150,150) -- (260,150) -- (260,145) -- (360,145) ;
	%Straight Lines [id:da8172413592821064] 
	\draw    (150,170) -- (150,103) ;
	\draw [shift={(150,100)}, rotate = 90] [fill={rgb, 255:red, 0; green, 0; blue, 0 }  ][line width=0.08]  [draw opacity=0] (8.93,-4.29) -- (0,0) -- (8.93,4.29) -- cycle    ;
	%Straight Lines [id:da14001944469599503] 
	\draw [color=niceRed  ,draw opacity=1 ][line width=2.25]    (150,120) -- (360,120) ;
	%Straight Lines [id:da1551836211266111] 
	\draw  [dash pattern={on 0.84pt off 2.51pt}]  (150,145) -- (260,145) ;
	
	% Text Node
	\draw (361,150) node [anchor=north west][inner sep=0.75pt]    {$T$};
	% Text Node
	\draw (262,150) node [anchor=north west][inner sep=0.75pt]    {$T/2$};
	% Text Node
	\draw (148,96.6) node [anchor=south east] [inner sep=0.75pt]    {$g^{\instb}( a)$};
	% Text Node
	\draw (148,120) node [anchor=east] [inner sep=0.75pt]    {$1$};
	% Text Node
	\draw (148,150) node [anchor=north east] [inner sep=0.75pt]    {$-\rho $};
	% Text Node
	\draw (148,145) node [anchor=east] [inner sep=0.75pt]    {$-\delta\rho$};

\end{tikzpicture}}
		%\caption{Expected \gft~of prices $(p,p+\delta)$.}
		% \caption{}
		% \label{fig:instancetwocosts}
	\end{subfigure}
	\caption{Lower bound adversarial setting: rewards and costs in the two instances \insta and \instb.}
	\label{fig:instancetwo}
\end{figure*}

\theoremtightCR*
\begin{proof}
	We show that, for all $\epsilon>0$ and $\delta\in(0,1)$, there exists two instances such that it is impossible to obtain $\E[V(T)]\le\epsilon T$ and 
	\[\frac{\optadv}{\E[\Rew]}\ge\frac{1+\rhoadv}{\rhoadv(1+\delta)+2\epsilon}\] in both instances.
	The two instances are denoted by \insta and \instb respectively, with $\cX=\{a_1, a_2\}$ and sequences of inputs of length $T$. The two instances are identical in the first $T/2$ rounds. 
	Rewards in instance \insta are, for each $t\in\range{T}$, $f^{\insta}_t(a_2)=0$ and $f^{\insta}_t(a_1)=\indicator{t\le T/2}$. On the other hand, in instance \instb we have $f^{\instb}_t(a_2)=0$, and $f^{\instb}_t(a_1)=1$ for all $t\in\range{T}$.
	Costs for the first instance \insta are define as
	\[
	g^{\insta}_t(a_1)\defeq\mleft\{\begin{array}{cc}  
		1 & \textnormal{ if } t\le T/2\\
		-1 & \textnormal{ otherwise }
	\end{array}
	\mright.,
	\]
	and $g^{\insta}_t(a_2)=-\rho$ for all $t\in\range{T}$.
	% while $g^{\insta}_t(a_2)=1$ if $t\le T/2$, while $g^I_t(a_2)=-1$ for $t\in\range{T/2, T}$.
	In the second instance $\instb$, costs are $g^{\instb}_t(a_1)=1$ for all $t\in\range{T}$, and 
	\[
	g^{\insta}_t(a_2)\defeq\mleft\{\begin{array}{cc}  
		-\rho & \textnormal{ if } t\le T/2\\
		-\delta\rho & \textnormal{ otherwise }
	\end{array}
	\mright.,
	\] 
	for some $\delta>0$. The two instances are depicted in \Cref{fig:instancetwo}.
	
	%	Also we will consider randomized algorithms, that produces at all $t$ randomized strategies $\xi_t\in\Delta(\cX)$.
	
	Let $N$ be the expected number of times that action $a_1$ is played in rounds $\range{T/2}$, that is
	\[
	N \defeq \sum\limits_{t\in \range{T/2}} \E^{\insta}[x_t=a_1]=\sum\limits_{t\in \range{T/2}} \E^{\instb}[x_t=a_1],
	\]
	where expectation is with respect to the algorithm's randomization. We observe that the algorithm plays in the same way in both instances up to time $T/2$, as they are identical (formally, the KL between instance \insta and \instb is zero in the first $T/2$ rounds).
	Then, we have that the optimal action in instance \insta is to play deterministically action $a_1$. Therefore, $\optadv^{\insta}=T/2$. The expected reward in instance \insta comes only from the number of plays of $a_1$ in the first $T/2$ rounds: $\E^{\insta}[\Rew]=N$.
	On the other hand, call $M$ the expected number of times an algorithm plays action $a_1$ in the last $\range{T/2,T}$ rounds of instance \instb, that is
	\[
	M\coloneqq\sum\limits_{t\in\range{T/2, T}}\E^{\instb}[x_t=a_1].
	\]
	We have that, in order to have $\E^{\instb}[V(T)]\le \epsilon T$ violations in the second instance, we need to play $a_1$ a small number of times:
	\[
	M-\delta\rho\mleft(\frac{T}{2}-M\mright)+N-\rho\mleft(\frac{T}{2}-N\mright)\le \epsilon T,
	\]
	which yields
	\[
	N\le \frac{T(\rho(\delta+1)+2\epsilon)}{2(\rho+1)}.
	\]
	
	Then, we get that
	\[
	\frac{\optadv^{\insta}}{\E^{\insta}[\Rew]}\ge \frac{1+\rho}{\rho(1+\delta)+2\epsilon},
	\]
	which concludes the proof since $\rhoadv^{\insta}=\rho$.
\end{proof}

\propRegretNotEnought*
\begin{proof}
	Consider the instance described in \Cref{ex:instanceone}, and consider an algorithm $\AlgP$ for $\cX=\{a_1,a_2,a_3\}$ such that $x_t=a_3$ for $t\in\range{T/3}$, while $x_t=a_2$ for $t\in\range{T/3,T}$. 
	Moreover, consider an algorithm $\AlgD$ instantiated on $\cD=[0,M]$, with $M\ge 1/\rho$, that plays $\lambda_t=0$ for all $t\in\range{2T/3}$, and $\lambda_t=M$ for all $t\in\range{2T/3, T}$. 
	
	We start by analyzing the primal regret achieved by $\AlgP$:
	\begin{align*}
		\RP_{T}&\coloneqq\sup_{x\in\cX}\sum_{t\in\range{T}}\hspace{-.2cm} \left[f_t(x)-f_t(x_t)-\lambda_t(g_{t,1}(x)-g_{t,1}(x_t))\right]\\
		&=\sup_{x\in\cX}\sum_{t\in\range{T}} \left[f_t(x)- \lambda_t g_{t,1}(x)\right]-\frac{2}{3}T+ \frac{M\rho}{3}T\\
		&=\sum_{t\in\range{T}} \left[f_t(a_1)-\lambda_tg_{t,1}(a_1)\right]+\frac{T}{3}\left(M\rho-2\right)\\
		&=\rho M\frac T3+\frac{T}{3}\left(M\rho-2\right)\\
		&=\frac{T}{3}\left(2M\rho-2\right)\le 0,\\
	\end{align*}
	where we replaced the $\sup$ with the utility at $a_1$ since $M\ge 1/\rho$.
	Moreover, the dual regret is such that  
	\begin{align*}
		\RD_{T}&\coloneqq\sup_{\lambda\in[ 0,M]}\sum_{t\in\range{2T/3, T}}\left(\lambda-M\right)g_{t,1}(x_t)\\
		&=\sup_{\lambda\in[ 0,M]}\frac{T}{3}\left(\lambda-M\right)\rho=0.
	\end{align*}
	However, for a suitable choice of $\rho$, the violations are linear in $T$ since
	\[
	V_1(T)\defeq\sum\limits_{t\in\range{T}}g_{t,1}(x_t)=\frac{\rho}{3} T =\Omega(T).
	\] 
	This concludes the proof.
\end{proof}

\section{Proof of \Cref{thm1}}

We start by providing the following auxiliary lemmas.

\begin{lemma}\label{lem:ogd_boundediterations}
	Let $\vy_t\in\mathbb{R}^m_{\ge 0}$ be generated by $\OGD$ with learning rate $\eta$ and utilities $\vy\mapsto\langle\vy, \vg_t\rangle$, where $\|\vg_t\|_\infty\le 1$ for all $t\in\range{T}$. Then:
	\[
	|\|\vy_{t+1}\|_1-\|\vy_t\|_1|\le m\cdot\eta
	\]
\end{lemma}

\begin{proof}
	The update of the $i$-th component of $\vy_{t+1}$ can be written as:
	\[
	y_{t+1,i}\coloneqq\max(0, y_{t,i}+\eta g_{t,i}).
	\]
	If $g_{t,i}\ge 0$ then the update can be simplified to $y_{t+1,i}= y_t+\eta g_{t,i}\le y_t+\eta$. If $g_{t,i}< 0$ then $y_{t+1,i}\ge y_{t,i}+\eta g_{t,i}\ge y_{t,i}-\eta$. Thus $|y_{t+1,i}-y_{t,i}|\le \eta$ for all $i\in\range{m}$. By summing over all component we have that $\|\vy_{t+1}-\vy_t\|_1\le m\cdot\eta$. By triangular inequality we have the desired statement.
\end{proof}

\begin{lemma}{[\cite[Chapter~10]{hazan2016introduction}]}\label{lem:adaptiveOGD}
	For any $t_1,t_2\in\range{T}$ with $t_1<t_2$, it holds that if $\vlambda_t$ is generated by $\OGD$ with learning rate $\eta>0$ on a set $\cD$, then:
	\[
	\RP_{\range{t_1,t_2}}(\{\vlambda\}) \le  \frac{\|\vlambda-\vlambda_{t_1}\|_2^2}{2\eta}+\frac{1}{2}\eta m T.
	\]
	with probability probability one on the randomization of the algorithm, \ie $\delta=0$. Moreover it also holds component-wise, \ie for all $\lambda\ge0$:
	\[
	\sum\limits_{t\in\range{t_1,t_2}}(\lambda-\lambda_t)g_t(x_t)\le \frac{(\lambda-\lambda_{t_1})^2}{2\eta}+\frac{1}{2}\eta T.
	\]
\end{lemma}

\begin{lemma}\label{lm:HoefEmpty}
%	\mat{Assume that the inputs are stochastic.} 
	In the stochastic setting, for any $\xi\in\Delta(\cX)$ and $\delta\in(0,1]$, with probability at least $1-\delta$, it holds that:
	\begin{align}
		& \sum_{t \in  I} \E_{x\sim \xi}\left[ \langle \vlambda_t, \vg_{t}(x) \rangle \right] \le \sum_{t \in  I} \E_{x\sim \xi}\left[  \langle \vlambda_t, \bar \vg_{t}(x) \rangle \right]+ M E_{T,\delta}\label{eq:HoefEmpty2} \quad\textnormal{and} \\
		& \sum_{t  \in  I}  \E_{x\sim \xi} \left[f_{t}(x) \right]  \ge  \sum_{ t \in  I} \E_{x\sim \xi} \left[ \bar f(x) \right]-    E_{T,\delta} ,\label{eq:HoefEmpty3}
	\end{align}
	for any interval $I=[t_1,t_2] \subseteq [T]$, where $E_{T,\delta}\defeq \sqrt{16T\log\left(\frac{2T}{\delta}\right)}$ and $M=\sup\limits_{t\in\range{T}}\|\vlambda\|_1$.
\end{lemma}

\begin{proof}
		We start by proving that the all the inequalities of Equation~\eqref{eq:HoefEmpty2} holds simultaneously with probability $1-\delta/2$.
		%consider the function $\pi$ that maps a time $t\in \cT_\nullx$ to its index. 
		We have that given a $I=[t_1,t_2]\subseteq [T]$, with probability at least $1-\nicefrac{\delta}{(2T^2)}$,
		\begin{align*}
		\sum_{t \in  I} \E_{x\sim \xi}\left[ \langle \vlambda_t, \vg_{t}(x) \rangle \right] - \sum_{t \in  I} \E_{x\sim \xi}\left[  \langle \vlambda_t, \bar \vg_{t}(x) \rangle \right] &\le M \sqrt{8|I|\log\left(\frac{2T^2}{\delta}\right)}\le M \sqrt{16T\log\left(\frac{2T}{\delta}\right)},
		\end{align*}
		where the first inequality holds by Azuma-Hoeffding inequality.
		By taking a union bound over all possible intervals $I$ (which are at most $T^2$), we obtain that all the first set of equations holdswith probability at least $1-\delta/2$. 
		
		Equation~\eqref{eq:HoefEmpty3} can be proved in a similar way. 
		Indeed, for any fixed interval $I=[t_1,t_2]\subseteq [T]$, and for any strategy mixture $\xi\in\Delta(\cX)$, by the Azuma-Hoeffding inequality we have that, with probability at least $1-\nicefrac{\delta}{(2T^2)}$, the following holds
		%it holds that, with probability $1-\nicefrac{\delta}{2T}$, for any fixed $IG$ of size $K=|IG|$, and for any strategy mixture $\xi$
		\begin{align*}
			  \sum_{ t \in  I} \E_{x\sim \xi} \left[\bar f(x) \right] - \sum_{t  \in  I } \E_{x\sim \xi}  \left[  f_t(x) \right] &\le \sqrt{2|I|\log\left(\frac{2T^2}{\delta}\right)}\le \sqrt{4T\log\left(\frac{2T}{\delta}\right)}.
		\end{align*}
		By taking a union bound over all possible $T^2$ intervals, we obtain that, for all possible intervals $I$, the equation above holds with probability $1-\delta/2$.
		
		The Lemma follows by a union bound on the two sets of equations above.
\end{proof}

These auxiliary technical lemmas are used in proving the following result.

\lemmaselfbounded* 

\begin{proof}
	Let $c_1\defeq 2$ and $c_2\defeq 12m$ and any learning rate $\eta$ for $\OGD$ with $\eta\le\eta_\OGD$. By contradiction, suppose there exists a time such that $\|\vlambda_t\|_1\ge c_2/\rho$, and let $t_2\in\range{T}$ be the smallest $t$ for which this happens.
	We unify the proof of the adversarial and stochastic setting. In particular, let $\rho=\rhoadv$ if the losses $(f_t,\vg_t)$ are adversarial, and let $\rho=\rhostoc$ if $(f_t,\vg_t)$ are stochastic with mean $(\bar f, \bar \vg)$. The extra stochasticity coming from the environment in the stochastic setting will be handled through \cref{lm:HoefEmpty}. 
	In order to streamline the notation, we define
	$E_{T,\delta}\defeq \sqrt{16T\log\left(\nicefrac{2T}{\delta}\right)}$. 
	
	Then, let $t_1\in\range{t_2}$ be the largest time for which $\|\vlambda_t\|_1\in[\frac{c_1}{\rho}, \frac{c_2}\rho]$ for all $t\in\range{t_1, t_2}$.
	
	\xhdr{Step 1.} First, we need to bound $\|\vlambda_{t_1}\|_1$ and $\|\vlambda_{t_2}\|_1$. To do that, we exploit \Cref{lem:ogd_boundediterations}.
%	
%	\Cref{lem:ogd_boundediterations} can be used to bound $\|\vlambda_{t_1}\|_1$ and $\|\vlambda_{t_2}\|_1$. 
	In particular, by telescoping the sum in the lemma, we obtain that:
	\[
	\|\vlambda_{t_2}\|_1-\|\vlambda_{t_1}\|_1\le \eta m(t_2-t_1).
	\]
	
	Moreover, by the definition of $\vlambda_{t_1}$ and $\vlambda_{t_2}$, we have:
	\[
	\frac{c_1}{\rho}\le\|\vlambda_{t_1}\|_1\le \|\vlambda_{t_1-1}\|_1+m\eta\le \frac{c_1}{\rho}+m\eta
	\]
	and similarly
	\[
	\frac{c_2}{\rho}\le\|\vlambda_{t_2}\|_1\le \|\vlambda_{t_2-1}\|_1+m\eta\le \frac{c_2}{\rho}+m\eta.
	\]
	This, together with the inequality above, yields
	\begin{equation}\label{eq:contradiction1}
	\frac{c_2-c_1}{2\eta m \rho}\le t_2-t_1.
	\end{equation}
	
	\xhdr{Step 2.} The range of the primal utilities in the turns $\range{t_1,t_2}$ can now be bounded as:
	\begin{align*}
	\sup\limits_{x\in\cX, t\in\range{t_1,t_2}} |\uP_t(x)|&\le \sup\limits_{x\in\cX, t\in\range{t_1,t_2}} \mleft\{|f_t(x)|+\|\lambda_t\|_1\cdot\|\vg_t(x)\|_\infty\mright\}\\
	&\le 1+\frac{c_2}{\rho}+m\eta\\
	&\le 1+\frac{12m+1}{\rho}\\
	&\le \frac{14m}{\rho}\eqqcolon L.
	\end{align*}
	Now, by the assumption that $\AlgP$ is weakly adaptive and $2$-scale-free, we obtain:
	\[
	\RP_{\range{t_1,t_2}}(\cX) \le L^2\cdot \overline{\RP}_{{T},\delta}(\cX),
	\]
	which holds with probability at least $1-\delta$.

	If we apply the primal no-regret condition above for strictly safe strategy $\xi^\circ\in\Delta(\cX)$ we have
	\begin{align}\label{eq:th11}
		\sum\limits_{t\in\range{t_1,t_2}}\cL_{f_t,\vg_t}(x_t,\vlambda_t) \ge \E_{x\sim \xi^\circ}\left[\sum\limits_{t\in\range{t_1,t_2}}\cL_{f_t,\vg_t}(x,\vlambda_t)\right]-L^2 \overline{\RP}_{T,\delta}(\cX).
	\end{align}
	
	Moreover, by definition of safe strategy we have that in the adversarial setting $\E_{x\sim\xi^\circ}[g_{t,i}(x)]\le-\rhoadv$ for all $i\in\range{m}$ and $t \in \range{t_1,t_2}$, while in the stochastic setting by Lemma~\ref{lm:HoefEmpty} it holds
	\[
	\sum_{t \in   \range{t_1,t_2}} \E_{x\sim \xi^\circ} \left[ \langle \vlambda_t, \vg_{t}(x) \rangle \right] \le \sum_{t \in   \range{t_1,t_2}} \E_{x\sim \xi^\circ} \left[ \langle \vlambda_t, \bar \vg_{t}(x) \rangle \right]+ M E_{T,\delta}\]
	and
	\[ \E_{x\sim\xi^\circ}[\bar g_{i}(\xi)]\le-\rhostoc \quad \forall i \in \range{m},
	\]
	where we recall that $E_{T,\delta}=\sqrt{16T\log\left(\nicefrac{2T}{\delta}\right)}$ and $M=\sup\limits_{t\in\range{T}}\|\vlambda\|_1$.
%	
%	Thus in both the stochastic and adversarial we can write that:
%	\[
%	\sum\limits_{t\in \range{t_1,t_2}}\langle\vlambda_t,\vg_t(x_t)\rangle\le -\rho\sum\limits_{t\in \range{t_1,t_2}}\|\vlambda_t\|_1+\left(\sup\limits_{t\in\range{t_1,t_2}}\|\vlambda_t\|_1\right)E_{T,\delta}. 
%	\]
	
	Therefore, we can lower bound the first term of the right-hand side of \Cref{eq:th11} the stochastic setting as:
%	\[
%	\mathbb{E}_{x\sim\xi^\circ}\left[\sum_{t\in\range{t_1,t_2}}g_{t,i}(x_t)\right]\le E_{T,\delta}-\rho(t_2-t_1)
%	\]
	\begin{align}
		 \E_{x\sim \xi^\circ}\left[\sum\limits_{t\in\range{t_1,t_2}}\cL_{f_t,\vg_t}(x,\vlambda_t)\right]&=\E_{x\sim \xi^\circ}\left[\sum\limits_{t\in\range{t_1,t_2}}f_t(x)-\langle\vlambda_t, \vg_t(x)\rangle\right]\nonumber\\
		 &\ge -  \E_{x\sim \xi^\circ}\left[\langle\vlambda_t, \vg_t(x)\rangle\right]\nonumber\\
		 &\ge -  \E_{x\sim \xi^\circ}\left[\langle\vlambda_t, \bar \vg(x)\rangle\right] - \left(\sup\limits_{t\in\range{T}}\|\vlambda\|_1\right)E_{T,\delta}   \nonumber\\
		 &\ge \rhostoc\sum\limits_{t\in\range{t_1,t_2}}\|\vlambda_t\|_1 - \left(\sup\limits_{t\in\range{T}}\|\vlambda\|_1\right)E_{T,\delta} \nonumber\\
		  &\ge \rhostoc\sum\limits_{t\in\range{t_1,t_2}}\|\vlambda_t\|_1-  \left(\frac{c_2}{\rhostoc}+m\eta\right) E_{T,\delta} \nonumber\\
		 &\ge c_1(t_2-t_1)- \left(\frac{c_2}{\rhostoc}+m\eta\right)  E_{T,\delta}\nonumber
	\end{align}
	In the adversarial setting we can more easily conclude that $\E_{x\sim \xi^\circ}\left[\sum\limits_{t\in\range{t_1,t_2}}\cL_{f_t,\vg_t}(x,\vlambda_t)\right]\ge c_1(t_2-t_1)$ and thus in both settings it holds that:
	\begin{equation}\label{eq:tmp1}
	 \E_{x\sim \xi^\circ}\left[\sum\limits_{t\in\range{t_1,t_2}}\cL_{f_t,\vg_t}(x,\vlambda_t)\right]\ge c_1(t_2-t_1)- \left(\frac{c_2}{\rhostoc}+m\eta\right)  E_{T,\delta}.
	\end{equation}
%	\mat{non si capisce come passiamo da $\rhoadv$ ... a $\rho$.}
	
	Combining the two inequalities of \Cref{eq:th11} and \Cref{eq:tmp1}, we can conclude that the overall utility of the primal algorithm $\AlgP$ can be lower bounded by:
	\begin{align}	\label{eq:pdTwo}
	\sum\limits_{t\in\range{t_1,t_2}}\uP_t(x_t)\ge c_1(t_2-t_1)-L^2 \overline{\RP}_{T,\delta}(\cX)-\left(\frac{c_2}{\rho}+m\eta\right)E_{T,\delta}
	\end{align}
	
	Now, we need an auxiliary result that we will use to upper bound the left hand side of the previous inequality.
	
	\begin{claim}\label{lem:aux_lemma_ogd} It holds that:
		\[
		\sum\limits_{t\in\range{t_1,t_2}}\langle\vlambda_t,\vg_t(x_t)\rangle\ge \frac{m}{2\rho^2\eta}.
		\]
	\end{claim}
	
%	Now we turn to prove that the existence of a time in which $\|\vlambda_t\|_1\ge ?$ implies an upper bound on the primal reward which, however, is lower then the lower bound just found. Thus reaching a contradiction.
	Then, we upper bound the left-hand side by using \Cref{lem:aux_lemma_ogd}:
	\begin{align}
	\sum\limits_{t\in\range{t_1,t_2}}\uP_t(x_t)=\sum\limits_{t\in\range{t_1,t_2}}\cL_{f_t,\vg_t}(x_t,\vlambda_t)&=\sum\limits_{t\in\range{t_1,t_2}} \left[f_t(x_t)-\langle\vlambda_t, \vg_t(x_t)\rangle\right] \nonumber\\
	& \le (t_2-t_1) - \frac{m}{2\rho^2\eta} \label{eq:pdOne}
	\end{align}
	
	Thus, combining Equation~\eqref{eq:pdOne} and~\eqref{eq:pdTwo}
	\[
	t_2-t_1\le\frac{1}{c_1-1}\left(L^2\overline{\RP}_{T,\delta}(\cX)- \frac{m}{2\rho^2\eta}+\left(\frac{c_2}{\rho}+m\eta\right)E_{T,\delta}\right).
	\]
	
	Combining it with \Cref{eq:contradiction1} one obtains that:
	\[
	\frac{c_2-c_1}{2\eta m\rho}\le \frac{1}{c_1-1}\left(L^2\overline{\RP}_{T,\delta}(\cX)- \frac{m}{2\rho^2\eta}+\left(\frac{c_2}{\rho}+m\eta\right)E_{T,\delta}\right),
	\]
	which gives as a solution $\eta\ge \frac{m^2-2\rho+13m\rho}{392m^3\overline{\RP}_{T,\delta}(\cX+2m\rho E_{T,\delta}(1+13m)}$.
	Which is a contradiction since:
	\[
	\eta\le\eta_\OGD\coloneqq  \frac{1}{800\cdot m \cdot  \max \left\{    {\overline{\RP}_{T,\delta}(\cX)}, E_{T,\delta}\right\} } > \frac{m^2-2\rho+13m\rho}{392m^3\overline{\RP}_{T,\delta}(\cX+2m\rho E_{T,\delta}(1+13m)}
	\]
	
	 Thus, we can conclude that $\|\vlambda_t\|_t\le c_2/\rho$ for each $t \in \range{T}$.
\end{proof}

Now, we provide the proof of \Cref{lem:aux_lemma_ogd}.

% \begin{lemma}\label{lem:aux_lemma_ogd}For any $t_1,t_2\in\range{T}$ we have	
% 	\[
% 	\sum\limits_{t\in\range{t_1,t_2}}\langle\vlambda_t,\vg_t(x_t)\rangle\ge\frac{m}{\rho^2\eta}.
% 	\]
% \end{lemma}

\begin{proof}[\textbf{Proof of \Cref{lem:aux_lemma_ogd}}]
	We define $\tilde t_i$ as the last time in $\range{t_1,t_2}$ in which $\lambda_{\tilde t_{i,1}}=0$, or $\tilde t_{1,i}=t_1$ if $\lambda_{t,i}>0$ for all $t\in\range{t_1,t_2}$. Formally:
	\[
	\tilde t_{1,i} = \max\left\{t_1, \sup_{\tau\in\range{t_2}: \lambda_{\tau,i}=0}\tau\right\}.
	\]
	We are now going to analyze separately for all $i\in\range{m}$, the rounds $\range{t_1, \tilde t_{1,i}}$ and the rounds $\range{\tilde t_{1,i}, t_2}$.
	
%	The dual regret with respect to any $\vlambda^*$ is:
%	\begin{align*}
%	\RD_{\range{t_1,t_2}}(\{\vlambda^*\})&\coloneqq\sum\limits_{t\in\range{t_1,t_2}} [\uD_t(\vlambda^*)-\uD_t(\vlambda_t)]\\
%	&=\sum\limits_{t\in\range{t_1,t_2}} \langle \vlambda^*-\vlambda_t, \vg_t(x_t) \rangle\\
%	&\le\overline{\RD}_{\range{t_1,t_2},\delta}(\{\vlambda^*\})
%	\end{align*}
%%	
%	Rearranging it gives:
%	\[
%	\sum\limits_{t\in\range{t_1,t_2}}\langle\vlambda_t,\vg_t(x_t)\rangle\ge	\sum\limits_{t\in\range{t_1,t_2}}\langle\vlambda^*,\vg_t(x_t)\rangle-\overline{\RD}_{\range{t_1,t_2},\delta}(\{\vlambda^*\})
%	\]
%%	
%	So we want a ``good'' dual strategy $\vlambda^*$ so that the right hand side is strictly greater then $0$.
%	We take $\vlambda^*$ defined as follows: $\lambda^*_i=\frac{1}{\rho}\mathbb{I}\left(\sum_{t\in\range{t_1,t_2}}g_{t,i}(x_t)\ge 0\right)$, and note that the definition of $\vlambda^*$ implies that the following condition holds:
%	\[
%	\sum\limits_{t\in\range{t_1,t_2}}\lambda_i^*g_{t,i}(x_t)=\frac{1}{\rho}\left[\sum\limits_{t\in\range{t_1,t_2}}g_{t,i}(x_t)\right]^+.
%	\]
%
%	
	\textbf{Phase 1:} First, we analyze the rounds $\range{t_1, \tilde t_{1,i}}$. 
By definition, it can be either that $\lambda_{\tilde t_{1,i}}=0$ or $\tilde t_{1,i}=t_1$. In the latter case, $\range{t_1,\tilde t_{1,i}}=\emptyset$ and the dual algorithm incurs zero regret.
In the former case, we can use \Cref{lem:adaptiveOGD} and write that the regret over the interval with respect to $\lambda_i^*=0$ is
\begin{equation}\label{eq:lemmaA33}
	0\le\sum\limits_{t\in\range{t_1,\tilde t_{1,i}}}\lambda_{t,i}g_{t,i}(x_t)+ \frac{\lambda_{ t_{1}}^2}{2\eta}+\frac{1}{2}\eta T\le \sum\limits_{t\in\range{t_1,\tilde t_{1,i}}}\lambda_{t,i}g_{t,i}(x_t)+ \frac{\lambda^2_{ t_{1}}}{2\eta}+\frac{1}{2}\eta T.
\end{equation}

	\textbf{Phase 2:} Now, we consider the rounds $\range{\tilde t_{1,i}, t_2}$. We take $\vlambda^*$ defined as follows: $\lambda^*_i=\frac{1}{\rho}$ for all $i\in\range{m}$.

	Let $\gaplambda_i\defeq \lambda_{t_2,i}-\lambda_{\tilde t_{1,i},i}$.
	Due to the definition of $\tilde t_{1,i}$, gradient descent never projects the multiplier relative to constraint $i$, and we can write that \[\sum\limits_{t\in\range{\tilde t_{1,i}, t_2}}g_{t,i}(x_t)=\frac{\gaplambda_i}{\eta}\] and, therefore,
	\begin{equation}\label{eq:lemmaA31}
	\sum\limits_{t\in\range{\tilde t_{1,i},t_2}}\lambda_i^*g_{t,i}(x_t)= \frac{\gaplambda_i}{\rho\eta}.
	\end{equation}

%	
%	and $\sum\limits_{t\in\range{t_1,t_2}}g_{t}(x_t)\le \frac{\lambda_{t_2,i}-\lambda_{t_1,i}}{\eta}$. 
	Now we can use \Cref{lem:adaptiveOGD} to find that:
	\[
	\sum\limits_{t\in\range{\tilde t_{1,i},t_2}}\lambda_i^*g_{t,i}(x_t)\le 	\sum\limits_{t\in\range{\tilde t_1,t_2}}\lambda_{t,i}g_{t,i}(x_t) + \frac{(\lambda_i^*-\lambda_{\tilde t_{1,i},i})^2}{2\eta}+\frac{1}{2}\eta T.
	\]
%	Combining everything and noting that $\lambda_{\tilde t_{1,i},i}\ge\frac{c_1}\rho$.
%	\[
%	\frac{1}{\rho\eta}\left(\frac{c_2-c_1}{\rho}+\eta \right)\le \sum\limits_{t\in\range{\tilde t_1,t_2}}\lambda_{t,i}g_t(x_t) + \frac{(\lambda_i^*-\lambda_{\tilde t_{1,i},i})^2}{2\eta}+\frac{1}{2}\eta T
%	\]
	Combining it with \Cref{eq:lemmaA31} yields the following
	\begin{equation}\label{eq:lemmaA32}
	\sum\limits_{t\in\range{\tilde t_{1,i},t_2}} \lambda_{t,i} g_{t,i}(x_t)\ge \frac{\gaplambda_i}{\rho\eta}-\frac{(\lambda_i^*-\lambda_{\tilde t_{1,i},i})^2}{2\eta}-\frac{1}{2}\eta T.
	\end{equation}
%	Now if we take $\eta\le \frac1{\sqrt{T}}$ then $\frac12\eta T\le \frac1{2\eta}$.
%	then we can conclude that:
%	\[
%	\sum\limits_{t\in\range{\tilde t_1,t_2}} \lambda_{t,i} g_{t,i}(x_t)\ge \frac{1}{\rho\eta}\left(\lambda_{t_2,i}-\lambda_{\tilde t_{1,i},i}\right)-\frac{(c_1+1+\eta)^2}{2\eta\rho^2}-\frac{1}{2\eta}.
%	\]
%	 

	Combining \Cref{eq:lemmaA32} and \Cref{eq:lemmaA33} we obtain:
	\begin{align*}
	\sum\limits_{t\in\range{t_1,t_2}} \lambda_{t,i}g_{t,i}(x_t)&\ge  \frac{\gaplambda_i}{\rho\eta}-\frac{(\lambda_i^*-\lambda_{\tilde t_{1,i},i})^2}{2\eta}-\frac{\lambda^2_{ t_{1}}}{2\eta}-\eta T\\
	&\ge  \frac{\gaplambda_i}{\rho\eta}-\frac{(\lambda_i^{*})^2+\lambda^2_{\tilde t_{1,i},i}}{2\eta}-\frac{\lambda^2_{ t_{1}}}{2\eta}-\eta T.%\tag{$(a-b)^2\le a^2+b^2$ if $a,b>0$}.
	\end{align*}

	Now, by summing over all $i\in\range{m}$, and by letting $\vlambda_{\tilde t_1}$ be the vector that has $\lambda_{\tilde t_{1,i}}$ as its $i$-th component, we get:	
	\begin{align*}
	\sum\limits_{t\in\range{t_1,t_2}}\langle\vlambda_t,\vg_t(x_t)\rangle&\ge \frac{\|\vlambda_{t_2}\|_1-\|\vlambda_{\tilde t_{1}}\|_1}{\rho\eta}-\frac{1}{2\eta}\left(\|\vlambda^*\|^2_2+\|\vlambda_{\tilde t_1}\|_2^2+\|\vlambda_{t_1}\|_2^2\right)-\frac{1}{\eta}\hspace{0.3cm}\textnormal{(as $\eta\le1/\sqrt T)$}\\[2mm]
	&\ge \frac{c_2}{\rho^2\eta}-\frac{1}{\rho\eta}\|\vlambda_{ t_1}\|_1-\frac{1}{2\eta}\left(\|\vlambda^*\|_2^2+2\|\vlambda_{t_1}\|_2^2\right)-\frac{1}{\eta}\tag{$\|\vlambda\|_1\ge c_2/\rho$ and $\|\vlambda_{\tilde t_1}\|_1\le \|\vlambda_{t_1}\|_1$}\\[2mm]
	&\ge \frac{c_2}{\rho^2\eta}-\frac{1}{\rho\eta}\left(\frac{c_1}\rho+m\eta\right)-\frac{1}{2\eta}\left(\frac{m}{\rho^2}+2\left(\frac{c_1}{\rho}+m\eta\right)^2\right)-\frac{1}{\eta}\nonumber \\[2mm]
	&\ge \frac{c_2}{\rho^2\eta} -\frac{c_1+1}{\rho^2\eta}-\frac{m}{2\rho^2\eta}	-\frac{2(c_1+1)^2}{2\rho^2\eta}	-\frac1\eta
	\tag{$\eta\le \nicefrac{1}{\rho m}$}\\[2mm]
	&\ge  \frac{2c_2-24-m}{2\rho^2\eta}\\
	&\ge\frac{m}{2\rho^2\eta}
	\end{align*}
where the last two inequalities hold due to the choice of parameters in the proof of \Cref{lem:aux_lemma_ogd}, that is $c_1=2$ and $c_2=13m$. This concludes the proof.
\end{proof}

\section{Omitted Proofs from \Cref{sec:general}}

\theoremviolations*
\begin{proof}
	The update of $\OGD$ for each component $i\in\range{m}$ is
	$\lambda_{t+1,i}:=[\lambda_{t,i}+\eta\vg_{t,i}(x_t)]^+$. Thus:
	\begin{align*}
		\lambda_{t+1,i}&\ge \lambda_{t,i}+\eta_\OGD g_{t,i}(x_t),
	\end{align*}
	and by induction:
	\[
	\lambda_{t+1,i}\ge \lambda_{0,i}+\eta_\OGD\sum\limits_{\tau=1}^t g_{\tau,i}(x_\tau).
	\]
	By rearranging and recalling that $\lambda_{0,i}=0$ we obtain:
	\[
	\sum\limits_{t\in\range{T}} g_{t,i}(x_t) \le \frac{1}{\eta_\OGD} \lambda_{T+1,i} \le\frac{1}{\eta} \lVert\vlambda_{T+1}\rVert_{1}
	\]
	Moreover, by \Cref{thm1} we can bound $\|\vlambda_T\|_1\le \frac{13m}{\rho}$ which holds with probability at least $1-\delta$. Thus, with probability at least $1-\delta$, it holds:
	\[
	V_T\coloneqq\max\limits_{i\in\range{m}}V_i(T)\le\frac{13m}{\eta_\OGD\rho}.
	\]
	The proof is concluded by observing that $\eta_\OGD=\tilde O\left((m\overline{\RP}_{T,\delta}(\cX))^{-1}\right)$.
\end{proof}

%\section{Omitted Proofs from \Cref{sec:general}}

\advregret*

\begin{proof}
	Define $x^*\in\cX$ such that:
	\[
	\sum\limits_{t\in\range{T}}f_t(x^*)=\optadv
	\]
	
	Now, consider a randomized strategy $\xi$ that randomized with probability $\alpha$ between $x^*$ and $\xi^\circ$, where $\xi^\circ$ is any strategy for which $\E_{x\sim\xi^\circ}[g_{t,i}(x_t)]\le-\rhoadv$. This strategy exists by assumption.
	Formally, for any $x\in\cX$ the randomized strategy $\xi$ assigns probability to $x$:
	\[
	\xi(x) = \alpha \delta_{x^*}(x) + (1-\alpha) \xi^\circ(x).
	\]
	
	Then, we compute the component of the primal utility of $\xi$ due to a constraint $i \in \range{m}$ as follows:
	\begin{align*}
		\E_{x\sim \xi}\left[ \sum\limits_{t\in\range{T}}\lambda_{t,i} g_{t,i}(x) \right]&=\alpha  \sum\limits_{t\in\range{T}}\lambda_{t,i} g_{t,i}(x^*)+(1-\alpha)\E_{x\sim \xi^\circ}\left[ \sum\limits_{t\in\range{T}}\lambda_{t,i} g_{t,i}(x) \right] \\
		&\le \alpha\sum\limits_{t\in\range{T}}\lambda_{t,i}-(1-\alpha)\rhoadv\sum\limits_{t\in\range{T}}\lambda_{t,i}\\
		&\le (\alpha-(1-\alpha)\rhoadv)\sum\limits_{t\in\range{T}}\lambda_{t,i}.
	\end{align*}
	Thus, setting $\alpha=\frac{\rhoadv}{1+\rhoadv}$ we have that $\E_{x\sim \xi}\left[ \sum_{t\in\range{T}}\lambda_{t,i} g_{t,i}(x) \right]\le0$, and $\sum\limits_{t\in\range{T}}\langle\vlambda_t, \vg_t(x_t)\rangle\le0$.
	
	We now compute the reward of $\xi$ for $\alpha=\frac{\rhoadv}{1+\rhoadv}$:
	\begin{align*}
		\E_{x\sim\xi}\left[\sum\limits_{t\in\range{T}} f_t(x)\right]&=\alpha \sum\limits_{t\in\range{T}}f_t(x^*)+(1-\alpha)\E_{x\sim\xi^\circ}\left[\sum\limits_{t\in\range{T}} f_t(x)\right]\\
		&\ge \frac{\rhoadv}{1+\rhoadv}\optadv
	\end{align*}
	
	Now, we consider the regret of $\AlgP$ with respect to $\xi$ and we find that:
	\[
	\sum\limits_{t\in\range{T}} \cL_{f_t,\vg_t}(x_t, \vlambda_t)\ge \mathbb{E}_{x\sim\xi}\left[\sum\limits_{t\in\range{T}} \cL_{f_t,\vg_t}(x, \vlambda_t)\right] - L^2\cdot\overline{\RP}_{T,\delta}(\cX).
	\]
	
	where $L$ is the maximum module of the payoffs of the primal regret minimizer, \ie $L\coloneq \sup_{t\in\range{T},x\in\cX}|\uP_t(x)|$.
	
	Exploiting the definition of $\cL_{f_t,\vg_t}(\cdot,\cdot)$ in the inequality above we obtain that:
	\begin{align}
		\sum\limits_{t\in\range{T}} f_t(x_t)-\langle \vlambda_t, \vg_t(x_t)\rangle&\ge \mathbb{E}_{x\sim\xi}\left[\sum\limits_{t\in\range{T}} f_t(x)-\langle \vlambda_t, \vg_t(x)\rangle\right]- L^2\cdot\overline{\RP}_{T,\delta}(\cX)\nonumber \\
		&\ge \mathbb{E}_{x\sim\xi}\left[\sum\limits_{t\in\range{T}} f_t(x)\right]- L^2\cdot\overline{\RP}_{T,\delta}(\cX)\nonumber \\
		&\ge  \frac{\rhoadv}{1+\rhoadv}\optadv- L^2\cdot\overline{\RP}_{T,\delta}(\cX)\label{eq:reqadv1}
	\end{align}
	Then, we lower bound the term $\sum\limits_{t\in\range{T}}\langle \vlambda_t, \vg_t(x_t)\rangle$ by using the dual regret of $\AlgD$ with respect to $\vlambda^*=\vzero$.
	Indeed, 
	\[
	\sum\limits_{t\in\range{T}}\langle\vlambda^*-\vlambda_t, \vg_t(x_t)\rangle \le \overline{\RD}_{T,\delta}(\{\vlambda^*\})
	\]
	implies that 
	\[
	\sum\limits_{t\in\range{T}}\langle\vlambda_t, \vg_t(x_t)\rangle \ge -\overline{\RD}_{T,\delta}(\{\vlambda^*\}).
	\]
	Combining it with \Cref{eq:reqadv1} gives:
	\[
	\sum\limits_{t\in\range{T}} f_t(x_t)\ge  \frac{\rhoadv}{1+\rhoadv}\optadv- L^2\cdot\overline{\RP}_{T,\delta}(\cX)-\overline{\RD}_{T,\delta}(\{\vlambda^*\}).
	\]
	Now, we use \Cref{thm1} which bounds $L\le 2 \frac{13m}{\rhoadv}$ and \Cref{lem:ogd_boundediterations} which we can use to bound $\overline{\RD}_{T,\delta}(\{\vlambda^*\})$.
	
	In particular, $\overline{\RD}_{T,\delta}(\{\vlambda^*\})$ can be bounded with:
	\[
	\overline{\RD}_{T,\delta}(\{\vlambda^*\})\le\frac{1}{2}\eta_\OGD m T,
	\]
	and thus:
	\[
		\Rew\coloneqq\sum\limits_{t\in\range{T}} f_t(x_t)\ge  \frac{\rhoadv}{1+\rhoadv}\optadv- 676\left(\frac m\rhoadv\right)^2\overline{\RP}_{T,\delta}(\cX)-\eta_\OGD m T.
	\]
	The proof is concluded by noting that $\eta_\OGD=\tilde O\left((m\overline{\RP}_{T,\delta}(\cX))^{-1}\right)$.
\end{proof}

\stocregret*
\begin{proof}
	
	By \Cref{thm1} we have that with probability at least $1-\delta$ we have that $\sup_{t\in\range{T}}\|\vlambda_t\|_1\le \frac{13m}{\rhostoc}$ and in the same way $\sup_{t\in\range{T},x\in\cX}\|\uP_t(x)\|_1\le 2\frac{13m}{\rhostoc}$.
	
	Define $\xi$ as the best strategy that satisfies the constraints, \ie $\optstoc:= T \ \E_{x\sim\xi}\left[\bar f(x)\right]$ and $\E_{x\sim\xi}[\bar g_{i}(x)]\le 0$.
	The no-regret property of $\AlgP$ with respect to $\xi$ gives that with probability $1-\delta$ it holds:
	\begin{align*}
	\sum\limits_{t\in\range{T}}[ f_t(x_t)-\langle\vlambda_t,\vg_t(x_t)\rangle]&\\
	&\hspace{-3cm}\ge \E_{x\sim\xi}\left[	\sum\limits_{t\in\range{T}}[ f_t(x)-\langle\vlambda_t,\vg_t(x)\rangle]\right]-\left(2\frac{13m}{\rhostoc}\right)^2\overline{\RP}_{T,\delta}(\cX)\\
	&\hspace{-3cm}\ge \E_{x\sim\xi}\left[	\sum\limits_{t\in\range{T}}[ \bar f(x)-\langle\vlambda_t,\bar \vg(x)\rangle]\right]-676\left(\frac{m}{\rhostoc}\right)^2\overline{\RP}_{T,\delta}(\cX)-2\left(\frac{13m}{\rhostoc}\right) E_{T,\delta}\\
	&\hspace{-3cm}=T \ \optstoc-676\left(\frac{m}{\rhostoc}\right)^2\overline{\RP}_{T,\delta}(\cX)-\frac{26m}{\rhostoc} E_{T,\delta},
	\end{align*}
	where the second inequality follows from \cref{lm:HoefEmpty} with $M\coloneqq\frac{13m}{\rhostoc}$.

	Moreover, the no-regret property of the dual regret minimizer $\AlgD$, with respect to $\vlambda^*=\vzero$, gives that:
	\[
	\sum_{t\in\range{T}}\langle\vlambda^*-\vlambda_t,\vg_t(x_t)\rangle\le\frac{1}{2}\eta_\OGD mT.
	\]
	Finally, we can combine everything from which follows that:
	\[
	\Rew\ge \optstoc-676\left(\frac{m}{\rhostoc}\right)^2\overline{\RP}_{T,\delta}(\cX)-\frac{26m}{\rhostoc} E_{T,\delta}-\frac{1}{2}\eta_\OGD mT.
	\]
	The proof is concluded by observing that $\eta_\OGD=\tilde O\left((m\overline{\RP}_{T,\delta}(\cX))^{-1}\right)$ and $E_{T,\delta}=\tilde O(\sqrt{T})$
\end{proof}
\section{Proofs omitted from \Cref{sec:applications}}

\lemmaerrL*
\begin{proof}
	Consider the following inequalities:
	\begin{align*}
		\err(\cO_\cL)&\coloneqq \sum\limits_{t\in\range{T}} \left(\hat \cL_t(z_t, a_t)-\bar \cL(z_t, a_t)\right)^2\\
		&\le 2\sum\limits_{t\in \range{T}} \left(\hat f_t(z_t,a_t)-\bar f(z_t,a_t)\right)^2+2\sum\limits_{t\in \range{T}} \left(\langle\vlambda_t,\hat \vg_t(z_t,a_t)\rangle-\langle\vlambda_t\bar \vg(z_t,a_t)\rangle\right)^2\tag{By AM-GM inequality: $2ab\le a^2+b^2$ for $a,b\ge0$.}\\
		&=2\cdot \err(\cO_f)+2\sum\limits_{t\in \range{T}} \left(\langle\vlambda_t,\hat \vg_t(z_t,a_t)-\bar \vg(z_t,a_t)\rangle\right)^2\\
		&\le2 \cdot \err(\cO_f)+2\sum\limits_{t\in \range{T}} \|\vlambda_t\|_1^2\cdot\|\hat \vg_t(z_t,a_t)-\bar \vg(z_t,a_t)\|_\infty^2\tag{$\langle a,b\rangle\le \|a\|_1\cdot\|b\|_\infty$}\\
		&\le2\cdot \err(\cO_f)+2\left(\sup\limits_{t\in\range{T}}\|\vlambda_t\|_1\right)^2\cdot\sum\limits_{t\in \range{T}} \|\hat \vg_t(z_t,a_t)-\bar \vg(z_t,a_t)\|_\infty^2\\
		&\le 2\cdot \err(\cO_f)+2\left(\sup\limits_{t\in\range{T}}\|\vlambda_t\|_1\right)^2\cdot\sum\limits_{t\in \range{T}}\sum\limits_{i\in\range{m}} (\hat g_{t,i}(z_t,a_t)-\bar g_i(z_t,a_t))^2\\
		&= 2\cdot \err(\cO_f)+2\left(\sup\limits_{t\in\range{T}}\|\vlambda_t\|_1\right)^2\cdot\sum\limits_{i\in\range{m}} \err(\cO_i)
	\end{align*}
	which concludes the proof. 
\end{proof}

\theoremadaptiveContextual*

\begin{proof}
	Consider any interval $I=\range{t_1, t_2}\subseteq\range{T}$. Since the prediction error at each time $t$ is positive, one trivially has that:
	\[
	\sum\limits_{t\in\range{t_1,t_2}}\left(\hat \cL_t(z_t,a_t)-\bar \cL(z_t,a_t)\right)^2\le \err(\cO_\cL).
	\]
	Then, applying \Cref{lem:La} we have that:
	\[
	\sum\limits_{t\in\range{t_1,t_2}}\left(\hat \cL_t(z_t,a_t)-\bar \cL(z_t,a_t)\right)^2\le2 \err(\cO_f)+2\sup\limits_{t\in\range{T}}\|\vlambda_t\|_1^2 \sum_{i\in\range{m}}\err(\cO_i).
	\]
	Moreover, by the assumption on the errors of the oracles it holds that:
	\begin{equation}\label{eq:lemFoster1}
	\sum\limits_{t\in\range{t_1,t_2}}\left(\hat \cL_t(z_t,a_t)-\bar \cL(z_t,a_t)\right)^2\le 2m(1+\sup\limits_{t\in\range{T}}\|\vlambda_t\|_1^2)\overline{\err}.
	\end{equation}
	%which holds by probability $1-\delta/(2T^2)$ by union bound.
	
%	Then note that \Cref{alg:algPrimalContex} has no memory, \ie the decision at time $t$ does not depend on the history up to to $t-1$ but only.\ma{make it formal?}
	Note that we could pretend that the algorithm starts at any time $t_1\in\range{T}$, and the same analysis of \cite[Theorem~1]{foster2020beyond} would hold, as their algorithm behavior does not depend on its past behavior. Hence, the following holds:
	\begin{align*}
	\RP_{\range{t_1,t_2}}(\Pi)&:=\sup_{\pi\in\Pi}\sum\limits_{t\in\range{t_1,t_2}} [\uP_t(\pi)-\uP_t(\pi_t)]\\
	&\coloneqq\sup_{\pi\in\Pi}\sum\limits_{t\in\range{t_1,t_2}} [\cL_t(\pi(z_t))-\cL_t(\pi_t(z_t))]\\
	&=\sup_{\pi\in\Pi}\sum\limits_{t\in\range{t_1,t_2}} [\cL_t(\pi(z_t))-\cL_t(a_t)]\\
	&\le \frac{\eta_\term{P}}{2}\err(\cO_\cL)+4\eta_\term{P}\log\left(\frac{2T^2}{\delta}\right)+2K\frac{T}{\eta_\term{P}}+\sqrt{2T\log\left(\frac{2T^2}{\delta}\right)}
	\end{align*}
	which holds with probability $1-\delta/(T^2)$. 
	
	Thus, by an union bound, and combining it with \Cref{eq:lemFoster1} we obtain that:
	\[
	\RP_{\range{t_1,t_2}}(\Pi)\le {\eta_\term{P}}m(1+\sup\limits_{t\in\range{T}}\|\vlambda_t\|_1^2)\overline{\err}+4\eta_\term{P}\log\left(\frac{2T^2}{\delta}\right)+2K\frac{T}{\eta_\term{P}}+\sqrt{2T\log\left(\frac{2T^2}{\delta}\right)},
	\]
	which holds with probability $1-\delta/T^2$. Finally, by tuning $\eta_\term{P}=\sqrt{KT}$ and applying an union bound on all the $T^2$ possible intervals $\range{t_1,t_2}$, we obtain that with probability $1-\delta$ it holds that:
	\[
	\sup\limits_{I=\range{t_1,t_2}}\RP_{\range{t_1,t_2}}(\Pi)\le 504\cdot m\  \overline{\err} \ L^2 \log(T^2/\delta) \sqrt{KT}.
	\]
\end{proof}

\end{document}